\documentclass{article}

\usepackage{amsmath,amsfonts,bm}
\usepackage{pifont}

\newcommand{\vep}{\bm{\varepsilon}}
\newcommand{\lnorm}{\left\lVert}
\newcommand{\rnorm}{\right\rVert}

\usepackage{amsmath}
\usepackage{amssymb}
\usepackage{mathtools}
\usepackage{amsthm}
\usepackage{adjustbox}
\usepackage{rotating}

\theoremstyle{plain}
\newtheorem{theorem}{Theorem}[section]
\newtheorem*{theorem*}{Theorem}

\newtheorem{lemma}[theorem]{Lemma}
\newtheorem*{lemma*}{Lemma}
\newtheorem{corollary}{Corollary}[section]

\theoremstyle{definition}
\newtheorem{definition}{Definition}[section]
\newtheorem{assumption}{Assumption}[section]
\newtheorem{problem}{Problem}[section]

\theoremstyle{remark}

\RequirePackage{algorithm}
\RequirePackage{algorithmic}

\def\eqref#1{equation~\ref{#1}}

\def\1{\bm{1}}

\def\rvx{{\mathbf{x}}}

\def\rmA{{\mathbf{A}}}
\def\rmB{{\mathbf{B}}}

\def\rmP{{\mathbf{P}}}

\def\vh{{\bm{h}}}

\def\vs{{\bm{s}}}

\def\vx{{\bm{x}}}
\def\vy{{\bm{y}}}

\def\mA{{\bm{A}}}

\def\mI{{\bm{I}}}

\def\mX{{\bm{X}}}

\DeclareMathAlphabet{\mathsfit}{\encodingdefault}{\sfdefault}{m}{sl}
\SetMathAlphabet{\mathsfit}{bold}{\encodingdefault}{\sfdefault}{bx}{n}

\def\gC{{\mathcal{C}}}
\def\gD{{\mathcal{D}}}
\def\gE{{\mathcal{E}}}

\def\gG{{\mathcal{G}}}

\def\gN{{\mathcal{N}}}
\def\gO{{\mathcal{O}}}

\def\gQ{{\mathcal{Q}}}

\def\gS{{\mathcal{S}}}

\def\gU{{\mathcal{U}}}
\def\gV{{\mathcal{V}}}

\def\gY{{\mathcal{Y}}}

\newcommand{\E}{\mathbb{E}}

\DeclareMathOperator{\sign}{sign}

\usepackage{subfigure}

\PassOptionsToPackage{numbers, sort, square}{natbib}

\usepackage{enumitem}
\usepackage{multirow}
\usepackage{multicol}
\usepackage{wrapfig}
\usepackage{dsfont}
\usepackage{lipsum}
\usepackage[preprint]{neurips_2024}

\usepackage[utf8]{inputenc} 
\usepackage[T1]{fontenc}    
\usepackage{hyperref}       
\usepackage{url}            
\usepackage{booktabs}       
\usepackage{amsfonts}       
\usepackage{nicefrac}       
\usepackage{microtype}      
\usepackage{xcolor}         
\usepackage{graphicx}
\usepackage{wrapfig}

\title{Online Algorithm for Node Feature Forecasting in Temporal Graphs}

\author{%
Aniq Ur Rahman \quad  Justin P. Coon\\
University of Oxford, U.K.\\
\texttt{aniq.rahman@eng.ox.ac.uk}, \quad  
\texttt{justin.coon@eng.ox.ac.uk}
}

\begin{document}

\maketitle
\begin{abstract}
In this paper, we propose an online algorithm \texttt{mspace} for forecasting node features in temporal graphs, which captures spatial cross-correlation among different nodes as well as the temporal auto-correlation within a node. The algorithm can be used for both probabilistic and deterministic multi-step forecasting, making it applicable for estimation and generation tasks. Comparative evaluations against various baselines, including temporal graph neural network (TGNN) models and classical Kalman filters, demonstrate that \texttt{mspace} performs at par with the state-of-the-art and even surpasses them on some datasets. Importantly, \texttt{mspace} demonstrates consistent performance across datasets with varying training sizes, a notable advantage over TGNN models that require abundant training samples to effectively learn the spatiotemporal trends in the data. Therefore, employing \texttt{mspace} is advantageous in scenarios where the training sample availability is limited. 
Additionally, we establish theoretical bounds on multi-step forecasting error of \texttt{mspace} and show that it scales linearly with the number of forecast steps $q$ as $\gO(q)$. For an asymptotically large number of nodes $n$, and timesteps $T$, the computational complexity of \texttt{mspace} grows linearly with both $n$, and $T$, i.e., $\gO(nT)$, while its space complexity remains constant $\gO(1)$.
We compare the performance of various \texttt{mspace} variants against ten recent TGNN baselines and two classical baselines, \texttt{ARIMA} and the \texttt{Kalman} filter across ten real-world datasets. Additionally, we propose a technique to generate synthetic datasets to aid in evaluating node feature forecasting methods, with the potential to serve as a benchmark for future research. Lastly, we have investigated the interpretability of different \texttt{mspace} variants by analyzing model parameters alongside dataset characteristics to jointly derive model-centric and data-centric insights. 
\end{abstract}




\section{Introduction}
\label{sec:intro}
Temporal graphs  are a powerful tool for modelling real-world data that evolves over time. They are increasingly being used in diverse fields, such as recommendation systems \cite{gao2022graph}, social networks \cite{deng2019learning}, and transportation systems \cite{yu2017spatio}, to name a few. 
Temporal graph learning (TGL) can be viewed as the task of learning on a sequence of graphs that form a time series. The changes in the graph can be of several types: changes to the number of nodes, the features of existing nodes, the configuration of edges, or the features of existing edges. Moreover, a temporal graph can result from a single or a combination of these changes. The TGL methods can be applied to various tasks, such as regression, classification, and clustering, at three levels: node, edge, and graph \citep{longa_graph_2023}.

In this work, we focus on node feature forecasting, also known as node regression, where the previous temporal states of a graph are used to predict its future node features. In most temporal graph neural network (TGNN) models, the previous states are encoded into a super-state or dynamic graph embedding \cite{barros2021survey}, guided by the graph structure. This dynamic embedding is then used to forecast the future node features. While the TGNN models perform well, they could be more interpretable, as a direct relationship between the node features and the embeddings cannot be understood straightforwardly. Furthermore, most embedding aggregation mechanisms impose a strong assumption that the neighbours influence a node in proportion to their edge weight \cite{wang_dissecting_nodate}.

TGNN methods \cite{li_diffusion_2017, micheli_discrete-time_2022, wu_graph_2019, fang_spatial-temporal_2021, liu_graph-based_2023} typically involve a training phase where the model learns from training data and is then deployed on test data without further training due to computational costs. If the test data distribution differs from the training data, an offline model cannot adapt. Therefore, when dealing with time-series data, it is crucial to use a lightweight online algorithm that can adapt to changes in data distribution while also performing forecasts. Moreover, TGNN models are typically trained to forecast a predetermined number of future steps. If we want to increase the number of forecast steps, even by one, the model needs to be reinitialized and retrained.

Inspired by the simplicity of Markov models, we define the state of a graph at a given time in an interpretable manner and propose a lightweight model that can be deployed without any training. The algorithm is designed with a mechanism to prioritize recent trends in the data over historical ones, allowing it to adapt to changes in data distribution.


\paragraph{Contributions} The contributions of our work are summarized as follows:
\begin{itemize}
    \item We have proposed an online learning algorithm \texttt{mspace} for node feature forecasting in temporal graphs, which can sequentially predict the node features for $q \in \mathbb{N}$ future timesteps after observing only two past node features.
    \item The algorithm \texttt{mspace} can produce both probabilistic and deterministic forecasts, making it suitable for generative and predictive tasks.
    \item The root mean square error (RMSE) of $q$-step iterative forecast scales linearly in the number of steps $q$, i.e. ${\rm RMSE}(q) = \gO(q)$, (see Appendix~\ref{app:bound}).
    \item For asymptotically large number of nodes $n$, and timesteps $T$, the computational complexity of \texttt{msapce} grows linearly with both $n$, and $T$, i.e., $\gO(nT)$, while the space complexity is constant $\gO(1)$, (see Appendix~\ref{app:complexity}).
    \item We have compared the performance of different variants of \texttt{mspace} against ten recent TGNN baselines, and two classical baselines \texttt{ARIMA}, and \texttt{Kalman} filter.
    \item We have evaluated \texttt{mspace} on four datasets for single-step forecasting and six datasets for multi-step forecasting.
    \item In addition to the evaluation on ten real-world datasets, we have proposed a technique to generate synthetic datasets that can aid in a more thorough evaluation of node feature forecasting methods. The synthetic datasets have the potential to serve as benchmark for future research.
    \item We have investigated the interpretability of different \texttt{mspace} variants by analyzing the model parameters along with the dataset characteristics to jointly derive model-centric and data-centric insights, (see Appendix~\ref{app:interpret}).
    \item To facilitate the reproducibility of results, the code is made available \href{https://anonymous.4open.science/r/NeurIPS2024Submissionmspace/README.md}{here}.
\end{itemize}

\paragraph{Notation} We denote vectors with lowercase boldface $\vx$, and matrices and tensors with uppercase boldface $\mX$. Sets are written in calligraphic font such as $\gV, \gU, \gS$, with the exception of graphs $\gG$, and queues $\gQ$. The operator $\succ$ is used in two contexts: $\vx \succ \bm{0}$ is an element wise positivity check on the vector $\vx$, and $\rmA \succ \bm{0}$ indicates that the matrix $\rmA$ is positive definite. $\mathbb{I}( \cdot )$ is the indicator function, and  $[m] \triangleq \{1, 2, \cdots, m\}$ for any $m\in \mathbb{N}$. We denote the distributions of continuous variables by $p(\cdot)$, and of discrete variables by $P(\cdot)$. The statement $\vx \sim p$ means that  $\vx$ is sampled from $p$. The Hadamard product operator is denoted by $\odot$ while the Kronecker product operator is denoted by $\otimes$. The trace of a matrix $\rmA$ is written as ${\rm tr}(\rmA)$.

We denote the neighbours of a node $v$ for an arbitrary number of hops as $\gU_v$.  We introduce the operator $\langle \cdot \rangle$ to arrange the nodes in a set $\gU$ in ascending numerical order of the node IDs. When another set or vector is super-scripted with $\langle \gU \rangle$, the elements within that set or vector are filtered and arranged as per $\langle \gU \rangle$. 

\newpage
\section{Methodology}
\label{sec:problem}
\paragraph{Problem Formulation}
A discrete-time temporal graph is defined as $\{ \gG_t = (\gV, \gE, \mX_t) : t \in [T] \}$, where $\gV= [n]$ is the set of nodes,  $\gE \subseteq \gV \times \gV$ is the set of edges, and $\mX_t \in \mathbb{R}^{n \times d}$ is the node feature matrix at time $t$. The set of edges $\gE$ can  alternatively be represented by the adjacency matrix denoted as $\mA \in \{0, 1\}^{n \times n}$.
The node feature vector is denoted by $\vx_t(v) \in \mathbb{R}^{d}$ such that $\mX_t = \begin{bmatrix}\vx_t(v) \end{bmatrix}^\top_{v \in \gV}$, and we refer to the first-order differencing \cite{shumway_time_2017} of a node feature vector as \textcolor{teal}{\textbf{shock}}. For a node $v \in \gV$ we define the shock at time $t$ as $\vep_t(v) \triangleq \bm{x}_t(v) - \bm{x}_{t-1}(v)$. The shock of the nodes in an ordered set $\gU$ at time $t$ is denoted by $\vep_t^{\langle \gU \rangle} \in \mathbb{R}^{|\gU|d}$. The shock at time $t$ for an arbitrary set of nodes is $\vep_t$.

\begin{assumption}
    The shocks $\{ \vep_1 , \vep_2 , \vep_3 \cdots \vep_T \}$ is assumed to be sampled from a continuous-state Markov chain defined on $\mathbb{R}^{nd}$ such that $p\left(\vep_{t+1} \mid \vep_t , \vep_{t-1}. \cdots \right) = p\left(\vep_{t+1} \mid \vep_t \right)$. 
\end{assumption}
This is a weak assumption because a continuous-state Markov chain has infinite number of states. However, having infinite number of states makes it impossible to learn the transition kernel from limited samples without additional assumptions on the model. To circumvent this, \textit{linear dynamical systems} and \textit{autoregressive models} are used in the literature \cite{barber2012bayesian}.

Let $p(\vep' \mid \vep)$ denote the transition probability $\vep \rightarrow \vep'$ in a continuous-state Markov chain $\mathfrak{Z}_0$ defined over $\gC$. A discrete-state Markov chain $\mathfrak{Z}_1$ defined over finite $\gS$ with transition probability $\rmP_{\vs, \vs'}$ can be constructed from $p(\vep' \mid \vep)$ through a mapping $\Psi: \gC \rightarrow \gS$ as
\begin{align}
    \rmP_{\vs, \vs'} = \frac{\int\limits_{\gC}\int\limits_{\gC} p(\vep' \mid \vep) p(\vep)\, \mathbb{I}( \Psi(\vep) = \vs ) \, \mathbb{I}( \Psi(\vep') = \vs' ) \, d \vep \, d\vep'}{\int\limits_{\gC}\int\limits_{\gC} p(\vep' \mid \vep) p(\vep) \,  \mathbb{I}( \Psi(\vep) = \vs ) \, d \vep \, d\vep'}.
\end{align}
\begin{wrapfigure}[7]{r}{0.4\columnwidth}
    \vspace{-15pt}
    \centering
    \includegraphics[width=0.38\columnwidth]{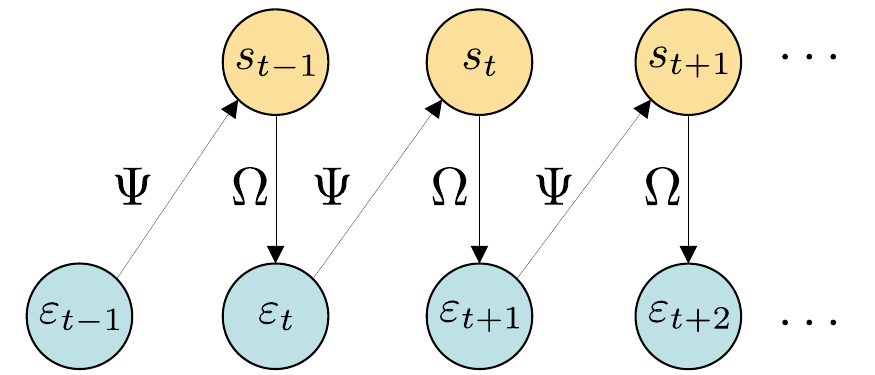}
    \caption{Markov Approximation.}
    \label{fig:state_shock}
\end{wrapfigure}
For a continuous-state Markov chain sample $\{ \vep_1, \vep_2, \cdots \vep_T \}$, we can estimate $\rmP$ directly from $\{ \Psi(\vep_1), \Psi(\vep_2), \cdots \Psi(\vep_T) \}$ without the need of $p(\vep' \mid \vep)$.
Consider a random function $\Omega: \gS \rightarrow \gC$, then the transition kernel can be defined over $\gS$ as:
\begin{align}
    \hat{\rmP}_{\vs, \vs'} = \int_{\gC} p(\Omega(\vs))\, \mathbb{I}(\Psi(\hat{\vep}')=\vs') \, d\vep'.
\end{align}

We refer to $\Psi$ as the \textcolor{teal}{\textbf{state function}}, and $\Omega$ as the \textcolor{teal}{\textbf{sampling function}}. The approximated Markov chain defined over $\gS$ resulting from $\Psi, \Omega$ is denoted as $\hat{\mathfrak{Z}}(\Psi, \Omega)$, with $p(\hat{\vep}' \mid \vep) = p(\Omega \circ \Psi(\vep))$.
Ideally, the goal is to find the pair of functions $(\Psi, \Omega)$ such that (a) $\rmP_{\vs, \vs'} = \hat{\rmP}_{\vs, \vs'} \, \forall \vs, \vs' \in \gS$, and (b)~$p(\vep' | \vep) = p(\hat{\vep}' | \vep)\, \forall \vep \in \gC$. In the absence of $p(\vep' | \vep)$, we can make use of markov chain samples.
\begin{problem}[$q$-step node feature forecasting]
    Let $\{ \vep_t : t \in [T] \} \sim \mathfrak{Z}_0$ be the sequence of shocks sampled from the original Markov chain, then for each $\vep_t$, a sequence of $q$ future shocks are generated from $\hat{\mathfrak{Z}}(\Psi, \Omega)$ with initial state $\vep_t$, i.e., $\{ \hat{\vep}_{t+k} : k \in [q] \}$ where $\hat{\vep}_{t+k} = \Omega \circ \Psi (\hat{\vep}_{t+k-1}), k>1$. The problem is to design $\Psi, \Omega$ such that $\min\limits_{\Psi, \Omega} \sum_{k\in[q]} \lnorm \sum_{j\in[k]} \vep_{t+j} - \hat{\vep}_{t+j} \rnorm^2 \, \forall t \in [T-q]$. 
\end{problem}
\paragraph{Proposed Model}
Instead of creating a single model to approximate $p(\vep_{t+1}^{\langle \gV \rangle} \mid \vep_{t}^{\langle \gV \rangle})$, we create a model for each node $v \in \gV$ to approximate $p(\vep_{t+1}^{\langle \gU_v \rangle} \mid \vep_{t}^{\langle \gU_v \rangle})$ where $\gU_v$ denotes the neighbours of node $v$ based on the following Assumption.
\begin{assumption}
    For nodes $v, u, u' \in \gV$, if $(v, u) \in \gE$ and $(v, u') \notin \gE$ then $\rho(\vx(v), \vx(u)) \geq 0$ and $\rho(\vx(v), \vx(u')) = 0$, where $\rho(\cdot, \cdot)$ is a correlation measure between two variables.
    \label{assum:edge1}
\end{assumption}
We propose two variants of the \textbf{state function}, one which captures the characteristics of the shock $\Psi_{\texttt{S}}$, and the other which is concerned with the timestamps $\Psi_{\texttt{T}}$ and captures seasonality.
\begin{itemize}
    \item $ \Psi_{\texttt{S}} : \mathbb{R}^{|\gU|d} \rightarrow \{ -1 , 1 \}^{|\gU|d}, \quad \Psi_{\texttt{S}}(\vep^{\langle \gU \rangle}) \triangleq \sign(\vep^{\langle \gU \rangle})$.
    \item $\Psi_{\texttt{T}}: \mathbb{N} \rightarrow \{ 0, 1, \cdots \tau_0-1 \}, \quad  \Psi_{\texttt{T}}( t ) \triangleq t \,{\rm mod}\, \tau_0$, where $\tau_0 \in \mathbb{N}$ is the time period.
\end{itemize}
We also define two variants of the \textbf{sampling function}:  probabilistic $\Omega_{\gN}$, and deterministic $\Omega_{\mu}$ where $\Omega_{\gN}(\vs) \triangleq \gN(\vep'; \bm{\mu}(\vs), \bm{\Sigma}(\vs)), \, \forall \vs \in \gS$, and $\Omega_{\mu}(\vs) \triangleq \bm{\mu}(\vs), \, \forall \vs \in \gS$.

\section{Algorithm}
\label{sec:algorithm}
We name our algorithm \textcolor{teal}{\texttt{mspace}} with a suffix specifying the state and sampling functions. For example, \texttt{mspace-S$\gN$} represents the algorithm with state function $\Psi_{\texttt{S}}$, and sampling function $\Omega_{\gN}$. For each node $v \in \gV$, we approximate $p( \vep_{t+1}^{\langle \gU_v \rangle} \mid \Psi_{\texttt{S}}(\vep_t^{\langle \gU_v \rangle}) = \vs )$ as a Gaussian distribution with mean vector $\bm{\mu}_v(\vs) \in \mathbb{R}^{|\gU_v|d}$ and covariance matrix $\bm{\Sigma}_v(\vs) \in \mathbb{R}^{|\gU_v|d\times |\gU_v|d}$ indexed by the state $\vs \in \{-1,1\}^{|\gU_v|d}$. The parameters $\bm{\mu}_v(\vs), \bm{\Sigma}_v(\vs)$ are learnt through maximum likelihood estimation (\texttt{MLE}). For each node $v \in \gV$, and state $\vs$ we maintain a \textcolor{teal}{\textbf{queue}} $\gQ_v(\vs)$ of maximum size $M$ in which the shocks succeeding a given state $\vs$ are collected. The \texttt{MLE} solution is calculated as $\bm{\mu}_v(\vs) \leftarrow {\rm mean}(\gQ_v(\vs))$, and $\bm{\Sigma}_v(\vs) \leftarrow {\rm covariance}(\gQ_v(\vs))$. 

\begin{wrapfigure}[10]{r}{0.45\columnwidth}
\vspace{-15pt}
    \centering
    \includegraphics[width=0.45\columnwidth]{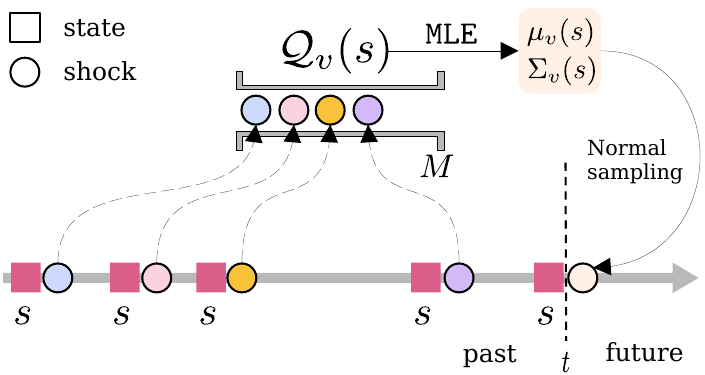}
    \caption{Operation of a queue.}
    \label{fig:markovfaux}
\end{wrapfigure}
The use of a fixed-size queue ensures that the model prioritises recent data over historical data, thereby allowing the system to adapt to prevailing trends. It must be noted that obtaining the parameters $\bm{\mu}_v(\vs), \bm{\Sigma}_v(\vs)$ from historical data \textit{relaxes the Markov assumption} in the original model. The queue $\gQ_v(\vs) = \{ \vep_{\tau}^{\langle \gU_V \rangle}: \Psi_{\texttt{S}}(\vep_{\tau-1}^{\langle \gU_v \rangle}) = \vs,  \tau < t  \}$ contains shocks from the past (see Fig.~\ref{fig:markovfaux}). Therefore, the estimated sample $\vep_{t+1}^{\langle \gU_v \rangle}$ depends on certain shocks from the past which violates the Markov property.

As \texttt{mspace} is an online algorithm, we might encounter unobserved states for which the queue is empty, and therefore cannot employ \texttt{MLE}. To facilitate \textit{inductive inference}, as a state $\vs_t$ is encountered, we find the state $\vs^* \in \gS_v$ which is the closest to $\vs_t$, i.e., $\vs^* \leftarrow \arg \min_{\vs \in \gS_v} \lnorm \vs - \vs_t \rnorm$, where $\gS_v$ is the set of states observed till time $t$.

\vspace{-5pt}
\begin{algorithm}[h!]
\label{alg:mspace_SN}
\caption{\texttt{mspace-S$\gN$}}
\begin{multicols}{2}
\begin{algorithmic}[1]
   \INPUT $\gG = (\gV, \gE, \mX)$, $r\in [0,1)$, $q, M$ 
   \OUTPUT $\hat{\vep}_t(v),$ $\forall v \in \gV, t \in [\lfloor r \cdot T \rfloor, T]$
   
   \STATE $\vep_t \leftarrow \vx_t - \vx_{t-1}, \quad \forall t \in [T]$

    \vspace{5pt}
   \textit{Offline training} $(\mathsf{A})$
   \FOR{$t \in [ \lfloor r \cdot T \rfloor ]$}
   \FOR{$v \in \gV$}
   \STATE $\displaystyle \vs_t \leftarrow \Psi_{\texttt{S}}\left( \vep_t^{\langle \gU_v \rangle}  \right)$
   \STATE $\gS_v \leftarrow \gS_v \cup \left\{ \vs_t  \right\}$
   \STATE $\gQ_v\left(\vs_t\right) \leftarrow$ enqueue $ \vep_{t+1}^{\langle \gU_v \rangle}$
   \ENDFOR
   \ENDFOR
   \vspace{5pt}
   \FOR{$v \in \gV$}
   \STATE   $\bm{\mu}_v(\vs)\leftarrow{\rm mean}(\gQ_v(\vs)), \, \forall s \in \gS_v$
   \STATE   $\bm{\Sigma}_v(\vs)\leftarrow{\rm covariance}(\gQ_v(\vs)), \, \forall s \in \gS_v$
   \ENDFOR

   \textit{Online learning}  $(\mathsf{B})$
   \FOR{$t \in [ \lfloor r \cdot T \rfloor , T-q ] $}
   \FOR{$v \in \gV$}
   \STATE $\displaystyle \vs_t \leftarrow \Psi_{\texttt{S}}\left( \vep_t^{\langle \gU_v \rangle}  \right)$
   \STATE $\displaystyle \vs^* \leftarrow \arg \min_{\vs \in \gS_v} \lnorm \vs - \vs_t \rnorm$
   \STATE $\hat{\vep}_{t+1}^{\langle \gU_v \rangle} \sim \gN( \vep; \bm{\mu}_v(\vs^*), \bm{\Sigma}_v(\vs^*)  )$
   \FOR{$k \in [q]\setminus \{1\}$}
   \STATE $\displaystyle \vs^* \leftarrow \arg \min_{\vs \in \gS_v} \lnorm \vs - \Psi\left( \hat{\vep}_{t+k-1}^{\langle \gU_v \rangle}  \right) \rnorm$
   \STATE $\hat{\vep}_{t+k}^{\langle \gU_v \rangle} \sim \gN( \vep; \bm{\mu}_v(\vs^*), \bm{\Sigma}_v(\vs^*)  )$
   \ENDFOR
   \STATE $\hat{\vep}_{t+k}(v) \leftarrow \hat{\vep}_{t+k}^{\langle \gU_v \rangle}(v), \quad \forall k \in [q]$
   \STATE Update $\gS_v, \gQ_v; \bm{\mu}_v(\vs), \bm{\Sigma}_v(\vs), \forall \vs \in \gS_v$
   \ENDFOR
   \ENDFOR
\end{algorithmic}
\end{multicols}
\vspace{-5pt}
\end{algorithm}
\begin{wrapfigure}[12]{l}{0.4\columnwidth}
\vspace{-7pt}
    \centering
    \includegraphics[width=0.4\columnwidth]{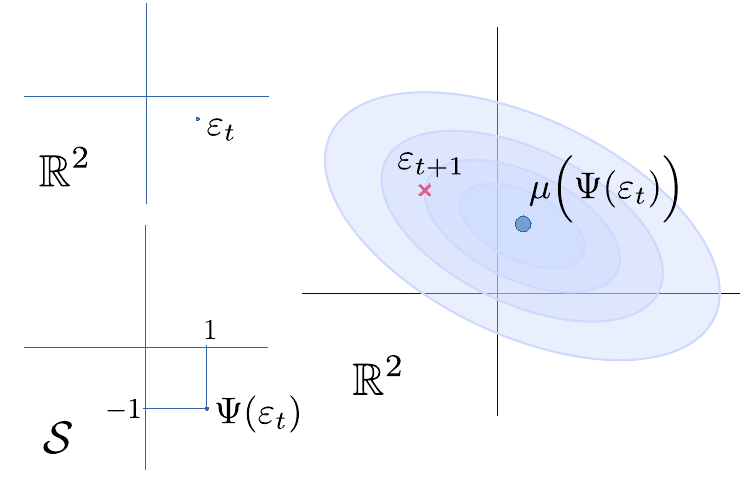}
    
    \caption{Shock Distribution.}
    \label{fig:sampling}
\end{wrapfigure}
\vspace{-10pt}
\paragraph{Example} For the purpose of explaining \texttt{mspace-S$\gN$} we consider an example with two nodes $n=2$, and feature dimension $d=1$. In Fig.~\ref{fig:sampling} we first show the shock vector $\vep_t \in \mathbb{R}^2$. The state of shock $\vep_t$, denoted by $\Psi(\vep_t)$ is marked in $\gS \in \{-1,1\}^2$. Corresponding to this state, we have a Gaussian distribution $\gN(\vep; \bm{\mu}(\Psi(\vep_t)) , \bm{\Sigma}(\Psi(\vep_t)))$ depicted as an ellipse. The next shock $\vep_{t+1}$ is sampled from this distribution. This distribution is updated as we gather more information over time. The volume of the Gaussian density in a quadrant is equal to the probability of the next shock's state being in that quadrant, i.e., the transition kernel $\hat{\rmP}_{\vs, \vs'} = \int_{\vs' \odot \vep \succ \mathbf{0}} \gN \left(\vep; \bm{\mu}(\vs), \bm{\Sigma}(\vs) \right) \, d\vep.$

\section{Related Works}
\label{sec:related}

\paragraph{Correlated Time Series Forecasting}
A set of $n$ time series data denoted as $\vx_t(v), \forall v \in [n], t \in [T]$ is assumed to be exhibit spatio-temporal correlation \cite{wu2021autocts, lai_lightcts_2023}. The correlations can then be discerned from the observations to perform forecasting.
The correlated time series (CTS) data can be viewed as a temporal graph $\gG = (\mX_t, \mA)$, with $\mX_t \triangleq \begin{bmatrix}
    \vx_t(v)
\end{bmatrix}_{v \in [n]}$ where the \textit{spatial correlation} between $\vx_t(u)$ and $ \vx_t(u)$ is quantified as the edge weight $\rmA_{u,v}$, and $\rmA_{u,u}$ signifies the \textit{temporal correlation} within $\vx_t(u)$.
The architecture of existing CTS forecasting methods consist of spatial (S) and temporal (T) operators. The S-operator can be a graph convolutional network (GCN) \cite{kipf_semi-supervised_2017} or a Transformer~\cite{vaswani_attention_2017}. As for the T-operator, convolutional neural network (CNN), recurrent neural network (RNN) \cite{chung_empirical_2014} or Transformer~\cite{zeng2023transformers} can be used.

\paragraph{Temporal Graph Neural Network}
A Graph Neural Network (GNN) is a type of neural network that operates on graph-structured data, such as social networks, citation networks, and molecular graphs. GNNs aim to learn node and graph representations by aggregating and transforming information from neighbouring nodes and edges \cite{wu_comprehensive_2021}. GNNs have shown promising results in various applications, such as node classification, link prediction, and graph classification.

\begin{wrapfigure}[8]{l}{0.4\columnwidth}
\vspace{-10pt}
    \centering
    \includegraphics[width=0.4\columnwidth]{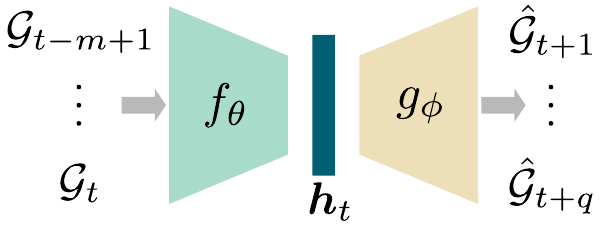}
    \caption{TGNN architecture.}
    \label{fig:tgnn-arch}
\end{wrapfigure}
Temporal GNN (TGNN) \cite{longa_graph_2023} is an extension of GNNs which operates on temporal graphs $\gG_t = (\mX_t, \mA_t)$ where $\mX_t$ denotes the node features, and $\mA_t$ is the evolving adjacency matrix. The TGNN architecture can be viewed as a neural network encoder-decoder pair $(f_{\theta}, g_{\phi})$ (Fig.~\ref{fig:tgnn-arch}). A sequence of $m$ past graph snapshots is first encoded into an embedding $\vh_t = f_{\theta} \big( \{ \gG_{t-m+1}, \cdots \gG_t \}  \big)$, and then a sequence of $q$ future graph snapshots is estimated by the decoder as $\{ \hat{\gG}_{t+1}, \cdots \hat{\gG}_{t+q} \} = g_{\phi}(\vh_t)$. The parameters $(\theta, \phi)$ are trained to minimize the difference between the true sequence $\{ \gG_{t+1}, \cdots \gG_{t+q} \}$ and the predicted sequence $\{ \hat{\gG}_{t+1}, \cdots \hat{\gG}_{t+q} \}$. In node feature forecasting, the objective is to minimize the difference between the node feature matrices $\{ \hat{\mX}_{t+1}, \cdots \hat{\mX}_{t+q} \}$ and $\{ {\mX}_{t+1}, \cdots {\mX}_{t+q} \}$, while in temporal link prediction, the goal is to minimize the difference between the graph structures $\{ \hat{\mA}_{t+1}, \cdots \hat{\mA}_{t+q} \}$ and $\{ {\mA}_{t+1}, \cdots {\mA}_{t+q} \}$.

There are two main approaches to implementing TGNNs: model evolution and embedding evolution. 
In \textit{model evolution}, the parameters of a static GNN are updated over time to capture the temporal dynamics of the graph, e.g., \texttt{EvolveGCN}~\cite{pareja_evolvegcn_2020}. 
In \textit{embedding evolution}, the GNN parameters remain fixed, and the node and edge embeddings are updated over time to learn the evolving graph structure and node features \cite{li_diffusion_2017, zhao_t-gcn_2019, micheli_discrete-time_2022, wu_graph_2019, fang_spatial-temporal_2021, liu_graph-based_2023}. The TGNN methods are described in Appendix~\ref{app:baselines}.

\begin{wrapfigure}[6]{r}{0.22\columnwidth}
\vspace{-25pt}
    \centering
    \includegraphics[width=0.22\columnwidth]{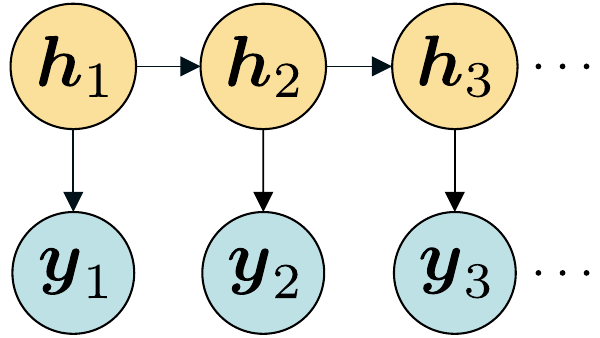}
    \caption{LDS.}
    \label{fig:lds}
\end{wrapfigure}
\paragraph{Linear Dynamical System}
In a linear dynamical system (LDS) \cite{barber2012bayesian}, the observation $\vy_t$ is modelled as a linear function of the latent vector $\vh_t$. The \textit{transition model} dictates the temporal evolution of the latent state $\vh_t = \rmA_t \vh_{t-1} + \bm{\eta}_t$, with $\bm{\eta}_t \sim \gN(\bm{\eta}; \bar{\vh}_t, \bm{\Sigma}_t)$, and the \textit{emission model} defines the relation between the observation and the latent state $\vy_t = \rmB_t \vh_t + \bm{\zeta}_t, \bm{\zeta}_t \sim \gN(\bm{\zeta}_t; \bar{\vy}_t, \bm{\Sigma}_t' )$. The LDS describes a first-order Markov model $p((\vy_t, \vh_t)_{t=1}^{T}) = p(\vh_1) p (\vy_1 \mid \vh_1) \prod_{t=2}^{T} p(\vh_t \mid \vh_{t-1}) p(\vy_t \mid \vh_t)$, where $p(\vh_t \mid \vh_{t-1}) = \gN(\vh_t; \rmA_t \vh_{t-1} + \bar{\vh}_t, \bm{\Sigma})$, and $p(\vy-t \mid \vh_t) = \gN(\vy_t; \rmB_t \vh_t + \bar{\vy}_t, \bm{\Sigma}_t')$. Therefore a LDS is defined by the parameters $(\rmA_t, \rmB_t, \bm{\Sigma}_t,  \bm{\Sigma}_t', \bar{\vh}_t, \bar{\vy}_t)$ and initial state $\vh_1$. In simplified models the parameters can be considered time-invariant. In the literature, LDS is also referred to as Kalman filter \cite{welch_introduction_1997}, or Gaussian state space model \cite{eleftheriadis2017identification}.

\paragraph{Gaussian Mixture Model}
A Gaussian mixture model (GMM) \cite{mclachlan_finite_2019} is a weighted sum of multiple Gaussian distribution components. An $M$-component GMM is defined as:
\begin{align}
   \textstyle p\left(\vx \right) = \sum_{i \in [M]} w_i \cdot \gN(\vx \,;\, \bm{\mu}_i, \bm{\Sigma}_i), \quad \sum_{i \in [M]} w_i = 1.
\end{align}
where $w_i$ denotes the probability of the sample belonging to the $i^{\rm th}$ component.
The parameters of the GMM  $ \{ (w_i, \bm{\mu}_i, \bm{\Sigma}_i) : \forall i \in [M] \}$  are learnt through \textit{expectation-maximisation} (EM) algorithm \cite[][Ch. 20, Sec. 3]{barber2012bayesian} , \textit{maximum a posteriori} (MAP) estimation, or \textit{maximum likelihood} estimation (MLE)  \cite[][Def. 8.30]{barber2012bayesian}.

\section{Results}
\label{sec:result}
\paragraph{Baselines} We compare the performance of \texttt{mspace} with the following recent TGNN baselines: \texttt{DCRNN}~\cite{li_diffusion_2017}, \texttt{TGCN}~\cite{zhao_t-gcn_2019}, \texttt{EGCN-H}~\cite{pareja_evolvegcn_2020}, \texttt{EGCN-O}~\cite{pareja_evolvegcn_2020}, \texttt{DynGESN}~\cite{micheli_discrete-time_2022}, \texttt{GWNet}~\cite{wu_graph_2019}, \texttt{STGODE}~\cite{fang_spatial-temporal_2021}, \texttt{FOGS}~\cite{rao_fogs_2022}, \texttt{GRAM-ODE}~\cite{liu_graph-based_2023}, \texttt{LightCTS}~\cite{lai_lightcts_2023}.
Additionally, we also evaluate the performance of classic autoregressive method \texttt{ARIMA}~\cite{box_distribution_1970}, and the famous LDS, the Kalman filter~\cite{welch_introduction_1997}. We introduce two variants of the Kalman filter: \texttt{Kalman-$\vx$}, which considers the node features as observations, and \texttt{Kalman-$\vep$}, which operates on the shocks. For more details, please see Appendix~\ref{app:eval}.

\paragraph{Datasets} We use the datasets \texttt{tennis}, \texttt{wikimath}, \texttt{pedalme}, and \texttt{cpox} for single-step forecasting as they are relatively smaller in terms of number of nodes $n$ and samples $T$. For multi-step forecasting we use the larger traffic datasets \texttt{PEMS03}, \texttt{PEMS04}, \texttt{PEMS07}, \texttt{PEMS08}, \texttt{PEMSBAY}, and \texttt{METRLA}. The datasets \texttt{PEMS03/04/07/08} report traffic flow, while \texttt{PEMSBAY}, and \texttt{METRLA} report traffic speed. 
\begin{table}[h!]
    \centering
    \small
    \resizebox{\columnwidth}{!}{%
    \begin{tabular}{ccccccccccc}
    \toprule
         & \texttt{tennis} & \texttt{wikimath} & \texttt{pedalme} & \texttt{cpox} & \texttt{PEMS03} & \texttt{PEMS04} & \texttt{PEMS07} & \texttt{PEMS08} & \texttt{PEMSBAY} &\texttt{METRLA} \\
         \midrule
         $n$ & 1000 & 1068  & 15  & 20 & 358 & 307  & 883  & 170 & 325 & 207\\
         $T$ & 120 & 731  & 35 & 520 & 26K  & 17K  & 28K & 18K & 52K  & 34K\\
         \bottomrule
    \end{tabular}}
\end{table}

\vspace{-10pt}
\paragraph{Single-step Forecasting}
In Table~\ref{result:single}, we have single-step forecasting RMSE results for various models with training ratio $0.9$. The best result is marked \textbf{bold}, and the second-best is \underline{underlined}.

The models \texttt{DCRNN}, \texttt{ECGN}, and \texttt{TGCN} exhibit similar performance across all datasets, which may be attributed to their use of convolutional GNNs for spatial encoding. 
\texttt{Kalman-$\vep$} performs poorly across all datasets, indicating challenges in establishing a state-space relation for shocks. In contrast, \texttt{Kalman-$\vx$} performs notably well, outperforming other methods on \texttt{tennis} and \texttt{pedalme} datasets. 

\begin{wraptable}[14]{r}{0.5\textwidth}
\vspace{-15pt}
\centering
\caption{Single-step forecasting RMSE, \tiny{($M=20$)}.}
\resizebox{0.5\columnwidth}{!}{%
\begin{tabular}{@{}lcccccc@{}}
\toprule
\texttt{}& \texttt{tennis} & \texttt{wikimath} & \texttt{pedalme} & \texttt{cpox}  \\
\midrule
\texttt{DynGESN} & 150.41 & 906.85 & 1.25 & 0.95 \\
\textcolor{purple}{\texttt{DCRNN}} & 155.43 & 1108.87 & 1.21 & 1.05 \\
\textcolor{purple}{\texttt{EGCN-H}} & 155.44 & 1118.55 & 1.19 & 1.06 \\
\textcolor{purple}{\texttt{EGCN-O}} & 155.43 & 1137.68 & 1.2 & 1.07 \\
\textcolor{purple}{\texttt{TGCN}} & 155.43 & 1109.99 & 1.22 & 1.04 \\
\texttt{LightCTS} & 199.04 & \underline{319.47}  & 1.58  & \underline{0.84} \\
\texttt{GRAM-ODE} & 206.50 & 484.90 & 0.99 & 0.98 \\
\texttt{STGODE} & 172.16 & \textbf{279.87} & 0.91 & \textbf{0.83} \\
\midrule
\texttt{mspace-S$\mu$} & \underline{105.32} & 563.69 & \underline{0.86} & 1.58 \\
\texttt{mspace-S$\gN$} & 117.23 & 725.42 & 1.35 & 2.11 \\
\midrule
\texttt{Kalman-$\vx$} & \textbf{73.01} & 792.6 & \textbf{0.66} & 1.42\\
\texttt{Kalman-$\vep$} & 7.5K & 64K & 1.79 & 10.2\\
\bottomrule
\end{tabular}
}
\label{result:single}
\end{wraptable}
For \texttt{wikimath} and \texttt{cpox}, \texttt{STGODE} shows the best performance, followed by \texttt{LightCTS} and \texttt{GRAM-ODE}, potentially due to a higher number of training samples. The light-weight methods such as \texttt{Kalman-$\vx$} and \texttt{mspace} exploit the unavailability of enough training samples and perform better on \texttt{tennis} and \texttt{pedalme}.

We notice that \texttt{mspace-S$\mu$} achieves a balanced performance between TGNN models and \texttt{Kalman-$\vx$} across all datasets except for \texttt{cpox}. The subpar performance of \texttt{mspace-S*} on the \texttt{cpox} dataset may be attributed to the seasonal trend, given that it represents the weekly count of chickenpox cases.
\begin{wrapfigure}[12]{l}{0.58\columnwidth}
\vspace{-15pt}
    \centering
    \subfigure[Training set.]{\includegraphics[width=0.28\columnwidth]{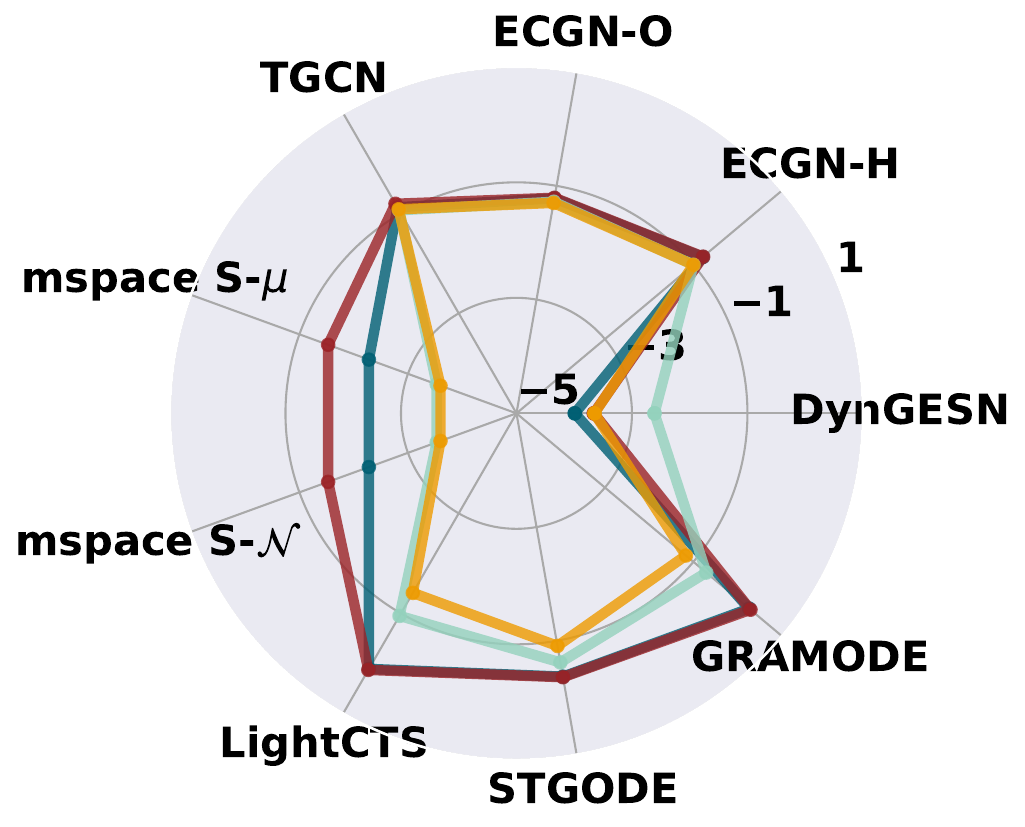}\label{fig:traingpu}}
    \subfigure[Test set.]{\includegraphics[width=0.28\columnwidth]{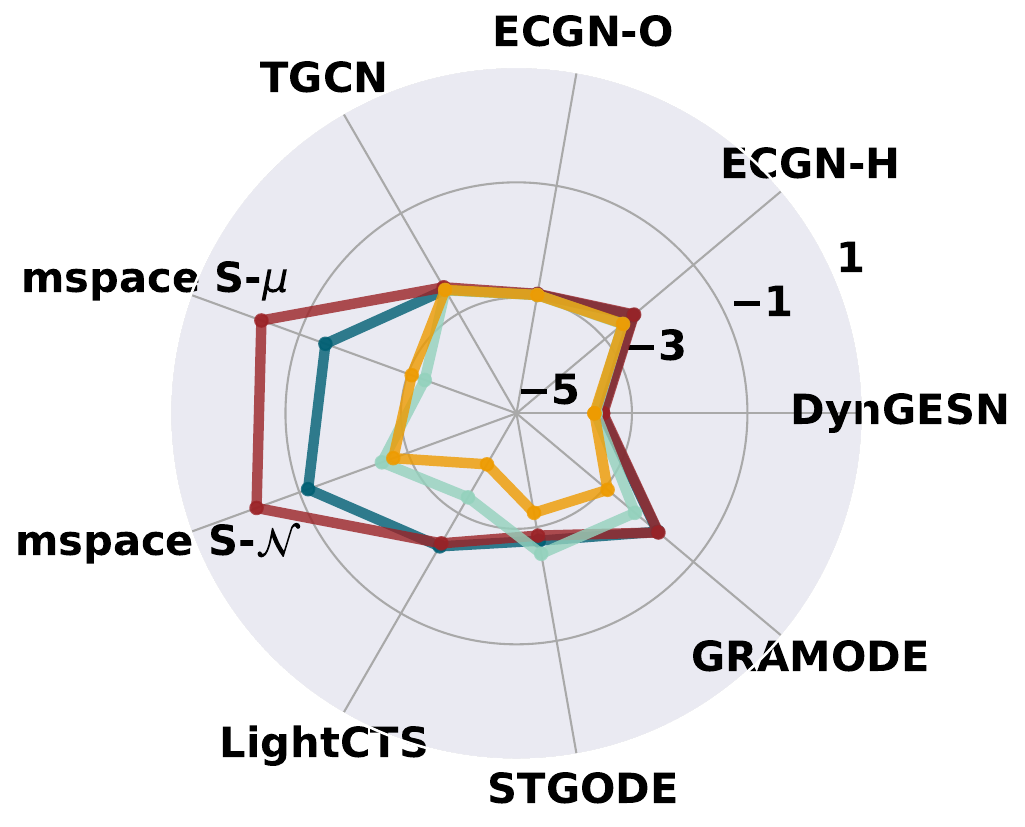} \label{fig:testgpu}}
    \vspace{-3pt}
    \caption{Average execution time per time-step for the datasets: 
    \textcolor[HTML]{005F73}{\texttt{tennis}}, 
    \textcolor[HTML]{9b2226}{\texttt{wikimath}},
    \textcolor[HTML]{94d2bd}{\texttt{pedalme}}, and
    \textcolor[HTML]{ee9b00}{\texttt{cpox}}.
    }
    \label{fig:radarplots}
\end{wrapfigure}

In Fig.~\ref{fig:radarplots}, we report the \textcolor{teal}{\textbf{average execution time per time-step}} ($\log_{10}$ scale) for the training as well as the test set. The TGNN models are run on the GPU, while \texttt{mspace} ran on the CPU\footnote{\textbf{GPU}: NVIDIA GeForce RTX™ 3060. \textbf{CPU}: 12th Gen Intel® Core™ i7-12700 × 20; 16.0 GiB.}.

For most TGNN models the execution time is similar across all datasets, whereas for \texttt{mspace-S*}, the execution time grows with the number of nodes as expected (see Algorithm~\ref{alg:mspace_SN}). The execution time of \texttt{mspace} can be reduced by parallelization.

The execution time of \texttt{mspace-S*} on the training set is lower than all the baselines except \texttt{DynGESN}. On the contrary, \texttt{mspace-S*} appears to be the slowest on the test set. This is because the baseline models only perform evaluation, while \texttt{mspace-S*} being an online algorithm also updates its parameters at every time-step.

\paragraph{Multi-step Forecasting} For the TGNN models, we use the 
$6:2:2$ train-validation-test chronological split in line with the experiments reported by the baselines. For \texttt{mspace} and \texttt{Kalman}, the train-test chronological split is $8:2$, as they do not require a validation set. In Table~\ref{result:multi_rmse_mae} we report the multi-step $q = 12$ forecasting RMSE, and mean absolute error (MAE) on the test set.

\begin{table}[h!]
\centering
\caption{Multi-step forecasting RMSE and MAE, \tiny{($M=20$)}.}
\resizebox{\columnwidth}{!}{%
\begin{tabular}{@{}lcccccccccccc@{}}
\toprule
\multirow{2}{*}{\texttt{}} & \multicolumn{2}{c}{\texttt{PEMS03}} & \multicolumn{2}{c}{\texttt{PEMS04}} & \multicolumn{2}{c}{\texttt{PEMS07}} & \multicolumn{2}{c}{\texttt{PEMS08}} & \multicolumn{2}{c}{\texttt{PEMSBAY}} & \multicolumn{2}{c}{\texttt{METRLA}}\\
\cmidrule(lr){2-3} \cmidrule(lr){4-5} \cmidrule(lr){6-7} \cmidrule(lr){8-9} \cmidrule(lr){10-11} \cmidrule(lr){12-13}
& \footnotesize{RMSE} & \footnotesize{MAE} & \footnotesize{RMSE} & \footnotesize{MAE} & \footnotesize{RMSE} & \footnotesize{MAE} & \footnotesize{RMSE} & \footnotesize{MAE} & \footnotesize{RMSE} & \footnotesize{MAE} & \footnotesize{RMSE} &\footnotesize{MAE} \\
\midrule
\texttt{GRAM-ODE} & \underline{26.40} & \underline{15.72} & 31.05 & 19.55 & \underline{34.42} & \underline{21.75} & 25.17 & 16.05 & \textbf{3.34} & \textbf{1.67}& \textbf{6.64} & \underline{3.44}\\
\texttt{STGODE} & 27.84 & 16.50 & 32.82 & 20.84 & 37.54 & 22.99 & 25.97 & 16.81 & 4.89 & 2.30& 7.37 & 3.75\\
\texttt{DCRNN} & 30.31 & 18.18 & 38.12 & 24.70 & 38.58 & 25.30 & 27.83 & 17.86 & 4.74 & 2.07& 7.60 & 3.60\\
\texttt{ARIMA} & 47.59 & 33.51 & 48.80 & 33.73 & 59.27 & 38.17 & 44.32 & 31.09 & 6.50 & 3.38& 13.23 & 6.90\\
\texttt{GWNet} & 32.94 & 19.85 & 39.70 & 25.45 & 42.78 & 26.85 & 31.05 & 19.13 & 4.85 & 1.95& 7.81 & 3.53\\
\texttt{LightCTS} & - & - & 30.14 & 18.79 & - & - & 23.49 & 14.63 & 4.32 & \underline{1.89}& \underline{7.21} & \textbf{3.42}\\
\texttt{FOGS} & \textbf{24.09} & \textbf{15.06} & 31.33 & 19.35 & \textbf{33.96} & \textbf{20.62} & 24.09 & 14.92 & - & -& - & -\\
\midrule
\texttt{mspace-S$\mu$} & 36.51 & 26.43 & \underline{18.85} & \underline{13.25} & 54.39 & 38.83 & \underline{14.61} & \underline{10.36} & 5.14 & 3.04& 10.64 & 7.13\\
\texttt{mspace-T$\mu$} & 26.53 & 18.31 & \textbf{13.49} & \textbf{8.70} & 38.63 & 24.02 & \textbf{10.35} & \textbf{6.33} & \underline{3.65} & 2.07& 9.22 & 6.14\\
\midrule
\texttt{Kalman-$\vx$} & 45.38 & 33.21 & 33.75 & 15.26 & 64.95 & 48.01 & 27.40 & 12.40 & 5.71 & 3.87& 13.97 & 10.7\\
\texttt{Kalman-$\vep$} & 749 & 619 & 818 & 709 & 2313 & 1988 & 460 & 399 & 50.2 & 43.1& 127.1 & 109\\
\bottomrule
\end{tabular}}
\label{result:multi_rmse_mae}
\end{table}

\begin{wrapfigure}[16]{r}{0.6\columnwidth}
\vspace{-10pt}
    \centering
    \includegraphics[width=0.59\columnwidth]{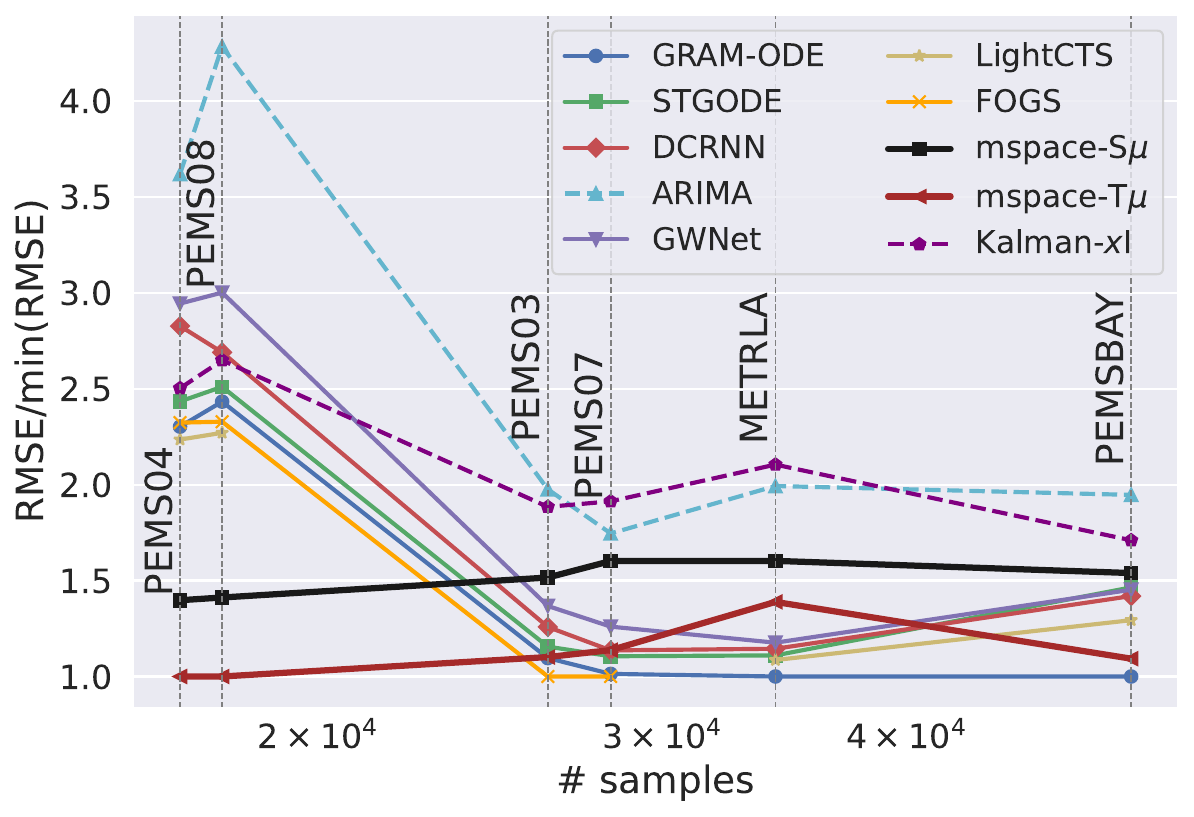} 
    \vspace{-10pt}
    \caption{Multi-step forecasting normalized RMSE.}
    \label{fig:rmse_trend}
\end{wrapfigure}
Figure~\ref{fig:rmse_trend} shows the RMSE of the models, normalized to the minimum RMSE for the dataset, plotted against the number of available training samples. We observe that \texttt{mspace-T$\mu$} performs competitively across all datasets with the exception of \texttt{METRLA}. Moreover, \texttt{mspace-T$\mu$} demonstrates superior performance compared to \texttt{mspace-S$\mu$} across all the datasets which suggests that temporal auto-correlation dominate spatial cross-correlation among the nodes. 

TGNN models, being neural networks, rely heavily on the amount of training data available. With the relatively small number of training samples in \texttt{PEMS04} and \texttt{PEMS08}, these models underperform. In contrast, both variants of \texttt{mspace} significantly surpass the state-of-the-art (SoTA), demonstrating their effectiveness with smaller datasets. Furthermore, \texttt{mspace-T$\mu$} ranks as the second-best model for the largest dataset, \texttt{PEMSBAY}. Therefore, we conclude that \texttt{mspace} offers consistent performance across datasets with varying sample sizes, and it is particularly advantageous when training data is limited.

\begin{wrapfigure}[12]{l}{0.5\columnwidth}
\vspace{-15pt}
    \centering
    \subfigure{\includegraphics[height=0.18\columnwidth]{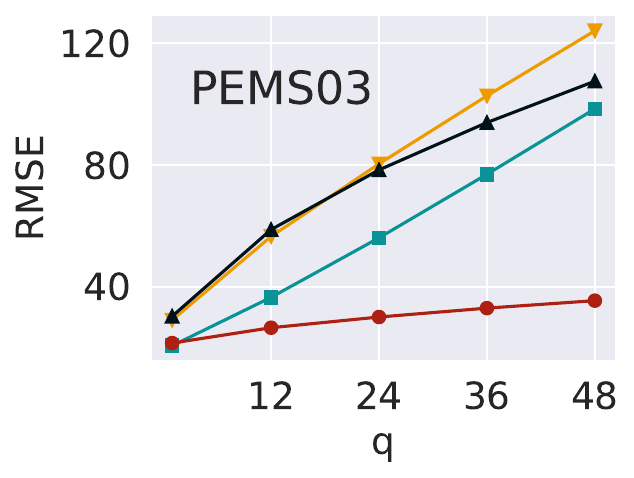}\label{fig:rmse03}}
    \subfigure{\includegraphics[height=0.18\columnwidth]{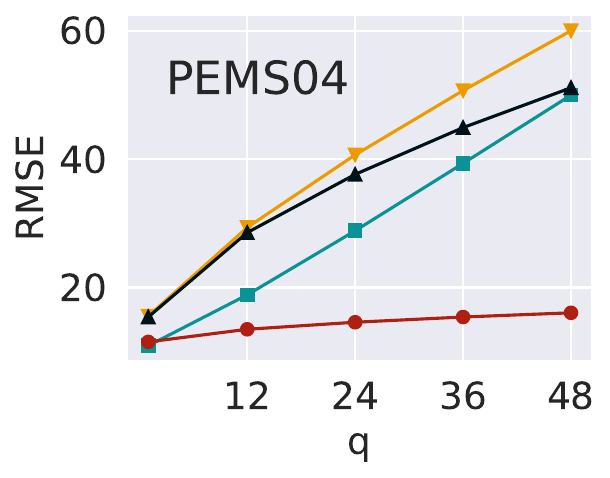}\label{fig:rmse04}}
    \caption{Scaling of error with the number of forecast steps $q$ using different \texttt{mspace} variants:
    \textcolor[HTML]{ee9b00}{$\blacktriangledown$~\texttt{mspace-S$\gN$}},
    \textcolor[HTML]{001219}{$\blacktriangle$~\texttt{mspace-T$\gN$}},
    \textcolor[HTML]{0a9396}{\tiny{$\blacksquare$}}~\textcolor[HTML]{0a9396}{\texttt{mspace-S$\mu$}},
    \textcolor[HTML]{ae2012}{$\bullet$~\texttt{mspace-T$\mu$}}.}
    \label{fig:bound_rmse}
    \label{fig:rmse_trend}
\end{wrapfigure}
In Fig.~\ref{fig:bound_rmse}, we illustrate how the RMSE scales with the number of forecast steps $q$ for different variants of \texttt{mspace}. The scaling law for \texttt{mspace-S*} appears linear, while for \texttt{mspace-T*}, it appears sublinear. We investigate this theoretically in Appendix~\ref{app:bound}. 

The TGNN baselines conduct forecasting for $q=12$ future steps, relying on the node features from the preceding $12$ time steps as input. In contrast, \texttt{mspace} requires only the node features from the two previous time steps. Additionally, \texttt{mspace} has the flexibility to forecast for any $q \in \mathbb{N}$, whereas TGNN models are limited to forecasting up to the specified number of steps they were trained on. Moreover, \texttt{mspace} offers both probabilistic ($\Omega_{\gN}$) and deterministic ($\Omega_{\mu}$) forecasts, a capability absent in the baselines. Finally, while TGNN baselines exploit the edge weights information for predictions, \texttt{mspace} achieves comparable results using only the graph structure.

\begin{wrapfigure}[20]{r}{0.5\columnwidth}
\resizebox{0.45\columnwidth}{!}{%
    \begin{minipage}{0.5\columnwidth}
    \vspace{-25pt}
      \begin{algorithm}[H]
        \caption{Synthetic Data Generation}
        \label{alg:syn_data}
        \begin{algorithmic}[1]
        \INPUT $\gG = (\gV, \gE)$, $d$, $\mu_{\min}$, $\mu_{\max}$, $\sigma^2_{\min}$, $\sigma^2_{\max}$, $\mu_0$, $\sigma^2_0$, $\tau$, $\mu_{\tau}$, $\sigma^2_{\tau}$. 
          \STATE $\vep_0 \sim {\rm Bernoulli}^{nd} \left( \frac{1}{2} \right) $
        \STATE $\vx_0 \sim \gN(\vx; \mu_0 \bm{1}, \sigma^2_0 \mI)$
        \FOR{ $t \in [T]$}
        \STATE $\vs_{t-1} \leftarrow \Psi_{\texttt{S}}(\vep_{t-1})$
        \IF{ $\vs_{t-1} \notin \gS $ }
        \STATE $\gS \leftarrow \gS \cup \{ \vs_{t-1} \}$  
        \STATE $\bm{\mu}(\vs_{t-1}) \sim {\rm Uniform}^{nd}(\mu_{\min}, \mu_{\max}) $
        \STATE $\Tilde{\bm{\Sigma}} \sim {\rm Uniform}^{nd\times nd}(\sigma^2_{\min}, \sigma^2_{\max}) $
        \STATE $\hat{\bm{\Sigma}} \leftarrow \frac{1}{2}\left( \Tilde{\bm{\Sigma}} + \Tilde{\bm{\Sigma}}^\top \right)$
        \STATE $\bm{\Sigma}(\vs_{t-1}) \leftarrow \hat{\bm{\Sigma}} \odot ( \mA \otimes \bm{1}_{d\times d} )$ 
        \ENDIF
        \STATE $\vep_t \sim \gN(\vep; \bm{\mu}(\vs_{t-1}), \bm{\Sigma}(\vs_{t-1}))$
        \STATE $\vx_t = \vx_{t-1} + \vep_t $
        \ENDFOR
        \IF{$\tau > 0$}
        \STATE $\vy_t \sim \gN(\vy; \mu_\tau \bm{1}, \sigma^2_{\tau} \mI ) \quad \forall t \in [\tau]$
        \STATE $\vx_t \leftarrow \vx_t + \vy_{t \,{\rm mod}\, \tau} \quad \forall t \in [T]$
        \ENDIF
        \end{algorithmic}
      \end{algorithm}
    \end{minipage}
    }
  \end{wrapfigure}
\paragraph{Synthetic Datasets}  In traffic datasets, seasonality outweighs cross-nodal correlation, making it challenging to assess the efficacy of a TGL algorithms on node feature forecasting task. To address this gap, we propose a synthetic dataset generation technique in line with the design idea of \texttt{mspace} which is described in Algorithm~\ref{alg:syn_data}.

In steps 8-10, we construct a covariance matrix adhering to Assumption~\ref{assum:edge1}, and in step 12, we sample the shock from a multivariate normal distribution. In steps 16-17, a random signal $\vy$ is tiled with period $\tau$ and added to the node features to introduce seasonality into the dataset.

The synthetic datasets can be utilized to analyze how various factors such as graph structure, periodicity, connectivity, sample size, and other parameters affect error metrics. The experiments on number of samples and data periodicity using synthetic datasets is deferred to Appendix~\ref{app:synthetic}.

\section{Analysis}
\label{sec:analysis}
\paragraph{Error Bounds}
We present the error bounds of \texttt{mspace} in the following theorem, a detailed proof of which can be found in Appendix~\ref{app:bound}.
\begin{theorem}
    The RMSE of \texttt{mspace} for a $q$-step node feature forecast is upper bounded as
    ${\rm RMSE}(q) \leq \sqrt{ \alpha q^2 + (3\alpha + \beta)q + \beta  }$, where $\alpha, \beta \in \mathbb{R}^+$ are constants that depend on the data, as well as the variant of the \texttt{mspace} algorithm.
    \label{thm:rmse}
\end{theorem}
\begin{corollary}
    In the asymptotic case of large $q$, the RMSE grows linearly with $q$: ${\rm RMSE}(q) = \gO(q)$.
\end{corollary}

\paragraph{Complexity Analysis} We denote the \textit{computational complexity} operator as $\mathfrak{C}(\cdot)$, and the \textit{space complexity} operator as $\mathfrak{M}(\cdot)$, where the argument of each operator is an algorithm or a portion of an algorithm. The optional offline part of \texttt{mspace} is denoted by $\mathsf{A}$, while the online part is denoted by $\mathsf{B}$. In Table~\ref{tab:complex_CM}, we exhibit the computational and space complexities of the different \texttt{mspace} variants, where $b \triangleq \max_{v \in [n]} |\gU_v|$ is the maximum degree. For more details please refer to Appendix~\ref{app:complexity}.
\begin{table}[h!]
    \centering
    \caption{Computational and space complexity of different \texttt{mspace} variants.}
    \label{tab:complex_CM}
    \resizebox{\columnwidth}{!}{%
    \begin{tabular}{l|l|l}
    
         &  \multicolumn{1}{c|}{$\Psi_{\texttt{S}}$} & 
         \multicolumn{1}{c}{$\Psi_{\texttt{T}}$}  \\ \bottomrule
         $\Omega_{\gN}$& 
         \parbox{8cm}{\begin{align*}
             &\textstyle \mathfrak{C}(\mathsf{A}) = \gO \left( ndb \left( rT + dbM  \min\{rT, 2^{bd}\} \right) \right) \\
             &\textstyle  \mathfrak{C}(\mathsf{B}) = \gO \left( (1-r)Tnd^2b^2 \left( qdb +  M  \min\left\{ \frac{(1+r)}{2}T, 2^{bd}   \right\}  \right)  \right)\\
             &\textstyle  \mathfrak{M}(\mathsf{A}\cup\mathsf{B}) = \gO \left( db(M + db)\min\{ T, 2^{bd} \}\right)
         \end{align*}}
         &
         \parbox{5cm}{\begin{align*}
             &\textstyle  \mathfrak{C}(\mathsf{A}) =  \gO\left( nrT + d^2Mn \tau_0   \right) \\
             &\textstyle  \mathfrak{C}(\mathsf{B}) = \gO\left( (1-r)Tn d^2  (  qd + M\tau_0 ) \right)\\
             &\textstyle  \mathfrak{M}(\mathsf{A}\cup\mathsf{B}) = \gO \Big( d(M + d) \tau_0    \Big)
         \end{align*}}
         \\
         \midrule
         $\Omega_{\mu}$&
         \parbox{8cm}{\begin{align*}
             &\textstyle \mathfrak{C}(\mathsf{A}) = \gO\left( ndb \left( rT + M  \min\{rT, 2^{bd}\} \right) \right) \\
             & \textstyle  \mathfrak{C}(\mathsf{B}) = \gO \left( (1-r)Tndb (q+M) \min \left\{ \frac{(1+r)}{2}T, 2^{bd}  \right\}  \right)\\
             &\textstyle  \mathfrak{M}(\mathsf{A}\cup\mathsf{B}) = \gO \left( Mdb  \min\{ T, 2^{bd} \}   \right)
         \end{align*}}
         & 
         \parbox{5cm}{\begin{align*}
             &\textstyle  \mathfrak{C}(\mathsf{A}) =  \gO\left( nrT + dMn \tau_0   \right) \\
             &\textstyle  \mathfrak{C}(\mathsf{B}) = \gO\left( (1-r)Tn  d (q+M)\tau_0 \right)\\
             &\textstyle  \mathfrak{M}(\mathsf{A}\cup\mathsf{B}) = \gO \left( Md \tau_0 \right)
         \end{align*}}
         \\
    \end{tabular}
    }
\end{table}
\begin{theorem}
    For asymptotically large number of nodes $n$ and timesteps $T$, the computational complexity of  \texttt{mspace} is $\gO(nT)$, and the space complexity is $\gO(1)$ across all variants.
    \label{thm:complex}
\end{theorem}
The proof is detailed in Appendix~\ref{app:asymp_complex}.

\newpage
\section{Discussion}
\label{sec:discuss}
In this section we discuss the limitations of \texttt{mspace} and how they can be overcome. Firstly, \texttt{mspace} only considers binary edges, i.e.. $\mA \in \{0,1\}^{n \times n}$ instead of a weighted adjacency matrix $\mA \in \mathbb{R}^{n \times n}$. This does not imply that we have used datasets with binary edges, rather it means that we have used a binarized version of the adjacency matrix as input to \texttt{mspace} while the baselines exploited weighted edges. Secondly, we assume that the graph structure is fixed throughout, while for a truly dynamic graph, the graph structure should also be dynamic. Lastly, we have proposed two state functions: one that focuses on cross-correlation among the nodes, and the other that considers seasonality. Therefore, a state function which  combines both can be studied in an extension of our work in the future.

\paragraph{On incorporating edge weights} We now investigate how we can incorporate edge weights in \texttt{mspace}, and if it has any potential benefits. In addition to Assumption~\ref{assum:edge1}, consider the following:
\begin{assumption}
    For nodes $v, u, u' \in \gV$, if $|\mA_{v, u}| \geq |\mA_{v, u'}|$ then $|\rho(\vx(v), \vx(u))| \geq |\rho(\vx(v), \vx(u'))|$, where $\mA_{v,u} \in \mathbb{R}$ denotes the edge weight of $(v,u) \in \gE$.
    \label{assum:weight}
\end{assumption}
Assumption~\ref{assum:edge1} can be applied to both weighted and unweighted graphs, while Assumption~\ref{assum:weight} is applicable only to weighted graphs. It is also evident that Assumption~\ref{assum:weight} is stronger than Assumption~\ref{assum:edge1}. Therefore, we base \texttt{mspace-S*} on Assumption~\ref{assum:edge1} and the correlation between the connected nodes are determined intrinsically through the conditional distributions, as the state $\Psi_{\texttt{S}}\left(\vep^{\langle \gU_v \rangle}\right)$ encodes the structural information of a node $v$ w.r.t its neighbours. However, we can enforce Assumption~\ref{assum:weight} through: $\vs^* \sim \left\{ \vs \in \gS_v : \lnorm 
\mA_v^{\langle \gU_v \rangle} \odot \left( \vs - \vs_t \right) \rnorm < \delta \right\}$, where $\delta \in \mathbb{R}^+$.

\paragraph{On adapting to dynamic graph structures} Algorithms that exploit dynamic graph structures are based on the temporal extension of Assumption~\ref{assum:weight}, formulated as:
\begin{assumption}
    For nodes $u,v \in \gV$, and time-steps $t, t' \in [T]$, if $|\mA_{u,v}(t)| \geq |\mA_{u,v}(t')|$, then $|\rho( \vx_t(u), \vx_t(v) )| \geq |\rho( \vx_{t'}(u), \vx_{t'}(v) )|$.
    \label{assum:dynamic}
\end{assumption}
Finding the matched state as $\vs^* \sim \left\{ \vs \in \gS_v : \lnorm 
\mA_v^{\langle \gU_v \rangle}(t) \odot \left( \vs - \vs_t^{\langle \gU_v \rangle} \right) \rnorm < \delta \right\}, \delta \in \mathbb{R}^+$ makes \texttt{mspace} compatible with dynamic graph structure. However, the number of nodes in the graph must remain fixed, i.e., \texttt{mspace} cannot deal with node addition or deletion.

\paragraph{On creating a state function which combines $\Psi_\texttt{S}$ and $\Psi_\texttt{T}$ } We can define $\Psi_{\texttt{ST}}: \mathbb{R}^{|\gU|d} \times \mathbb{N} \rightarrow \{ -1,1\}^{|\gU|d} \times \{ 0, 1, \cdots \tau_0-1 \}$ as $\Psi_{\texttt{ST}}\left(\vep^{\langle \gU \rangle}, t \right) \triangleq \begin{bmatrix} \sign(\vep^{\langle \gU \rangle})^\top & t \,{\rm mod}\, \tau_0  \end{bmatrix}^\top$.
In essence, the queues $\gQ_v(\vs), \forall \vs \in \gS_v, \forall v \in [n]$ in \texttt{mspace-ST} would have lesser entries compared to \texttt{mspace-S} which might lead to poor estimates and consequently make the algorithm data-intensive. Furthermore, in the step where we find the closest state $\vs^*$, the spatial and temporal parts can be assigned different weights: $\vs^* \leftarrow \arg \min_{\vs \in \gS_v} \lnorm 
\begin{bmatrix}
    \mathbf{1}_{d|\gU_v|} & \gamma
\end{bmatrix}^\top \odot \left( \vs - \vs_t^{\langle \gU_v \rangle} \right) \rnorm$, where $\gamma \in \mathbb{R}^+$.

\section{Conclusion}
\label{sec:conclusion}

In conclusion, our proposed algorithm, \texttt{mspace}, performs at par with the SoTA TGNN models across various temporal graph datasets. As an online learning algorithm, \texttt{mspace} is adaptive to changes in data distribution and is suitable for deployment in scenarios where training samples are limited. 
The interpretability of \texttt{mspace} sets it apart from black-box deep learning models, allowing for a clearer understanding of the underlying mechanisms driving predictions. This emphasis on interpretability represents a significant step forward in the field of temporal graph learning. In Sec.~\ref{sec:discuss}, we discussed the potential limitations of \texttt{mspace}, and suggested design changes through which they can be overcome.

In addition to the algorithm, we also introduced a synthetic temporal graph generator in which the features of the nodes evolve with the influence of their neighbours in a non-linear way. These synthetic datasets can serve as a valuable resource for benchmarking algorithms.

\newpage

\bibliography{references}
\bibliographystyle{ieeetr}
\setcitestyle{square}
\appendix


\newpage
\section{Interpretability}
\label{app:interpret}

In this section, we examine \texttt{mspace} in light of the following definition of Interpretability.
\begin{definition}
    Consider data $\vx \in \gD$ which is processed by a model $\mathsf{F}_{\theta}$ to produce the output $\hat{\vy} \in \gY$, i.e., $\hat{\vy} = \mathsf{F}_{\theta}(\vx)$, where $\theta$ denotes the model parameters. Moreover, consider a true mapping $f:\vx \mapsto \vy, \, \forall \vx \in \gD$ where $\vy$ is the ground truth associated with the input data $\vx$. Then, an interpretable or explainable model $\mathsf{F}_{\theta}$ fulfils one or more of the following properties \cite{gilpin2018explaining, du2019techniques}:
    \begin{itemize}
        \item  The internals of the model $\mathsf{F}_{\theta}$ can be explained in a way that is understandable to humans.
        \item The output $\hat{\vy}$ can be explained in terms of the properties of the input $\vx$, the input data distribution $\gD$, and the model parameters $\theta$.
        \item The failure of a model on a given input data can be explained.
        \item For a certain distance metric $\Delta: \gY \times \gY \rightarrow \mathbb{R}^+$, theoretical bounds on the expected error $\mathbb{E}_{\vx \sim \gD} [ \Delta(\vy, \mathsf{F}_{\theta}(\vx))] $ can be established based on the description of $\mathsf{F}_{\theta}$, supported by the assumptions on the input data distribution $\gD$.
        \item It can be identified whether the model $\mathsf{F}_{\theta}$ is susceptible to training bias, and to what extent.
    \end{itemize}
\end{definition}




\paragraph{Explaining $\Psi_{\texttt{S}}$} In Fig.~\ref{fig:intuitive}, we depict two consecutive snapshots of a subgraph, focused on node $v$. The dashed circle highlights the corresponding 1-hop neighbourhood $\gU_v$. At any time $t$, we draw green and red arrows to next to the nodes to depict whether its node feature value increased or decreased, respectively.

\begin{figure}[h!]
   \centering
    \includegraphics[width=0.5\columnwidth]{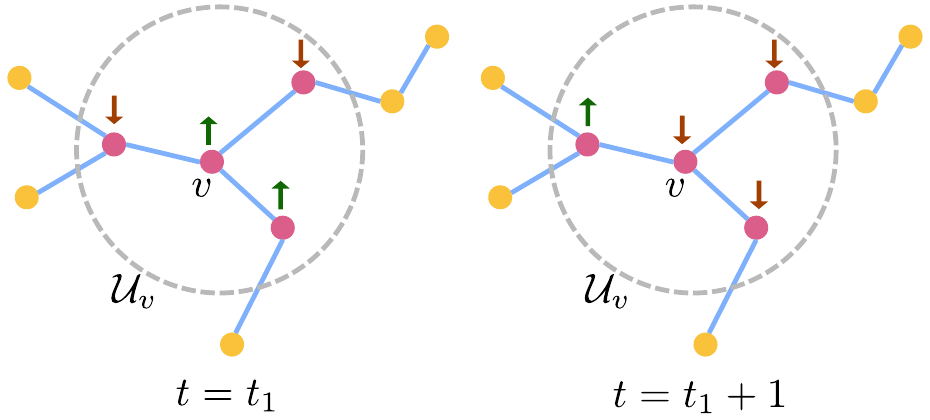}
    \caption{Consecutive subgraph snapshots.}
    \label{fig:intuitive}
\end{figure}
The design of $\Psi_{\texttt{S}}$ was inspired by the correlation dynamics of the stock market \cite{caraiani2014predictive}, where the inter-connectedness of various stocks exerts mutual influence on their respective prices. For instance, within the semiconductor sector, stocks such as \texttt{NVDA}, \texttt{AMD}, and \texttt{TSMC} often exhibit synchronised movements, with slight lead or lag. Similarly, the performance of gold mining stocks can offer insights into the future value of physical gold and companies engaged in precious metal trade. This concept transcends individual industries and encompasses competition across multiple sectors.

Let us record the states at two consecutive time-steps $\vs_{t_1} = \begin{bmatrix}
    1 & -1 & 1 & -1
\end{bmatrix}^\top$, and $\vs_{t_1+1} = \begin{bmatrix}
    -1 & -1 & -1 & 1
\end{bmatrix}^\top$. At the state-level, we iterate through the time-steps, and collect all the states succeeding $\vs = \begin{bmatrix}
    1 & -1 & 1 & -1
\end{bmatrix}^\top$. If we then draw a random sample from this collection of succeeding states, we can predict whether the node feature value is more likely to \textit{increase or decrease}. However, we are interested in predicting the \textit{amount} of change. Therefore, at every time step when the state $\vs_t$ matches $\vs = \begin{bmatrix}
    1 & -1 & 1 & -1
\end{bmatrix}^\top$, we collect the succeeding shock $\vep_{t+1}^{\langle \gU_v \rangle}$ in a queue $\gQ_v(\vs)$, i.e., at time $\tau$, $\gQ_v(\vs) = \left\{ \vep_{t+1}^{\langle \gU_v \rangle} : \vs_t = \vs, \forall t < \tau \right\}$ with $|\gQ_v(\vs)| \leq M$.
The queue entries are then used to approximate a distribution from which a random sample is drawn during forecast.

In Fig.~\ref{fig:S_chars}, we plot the normalized histogram of the trace ${\rm tr}(\cdot)$ of the covariance matrix $\bm{\Sigma}(\vs)$ of all the states $\vs \in \gS_v, v \in [n]$ for all the datasets used in multi-step forecasting. We notice that in both \texttt{PEMS04} and \texttt{PEMS08} the distribution of values is skewed to the left, with a concentration of data points at values close to zero. This explains the better-than-SoTA performance of \texttt{mspace-S$\mu$} on these datasets. In contrast, the histogram of \texttt{METRLA} is completely away from zero, while for \texttt{PEMS03}, and \texttt{PEMS07} there are peaks near zero, but a major mass of the histogram is skewed away from zero. This explains the poor performance of \texttt{mspace-S$\mu$} on these datasets.

\begin{figure}
    \centering
    \includegraphics[width=\columnwidth]{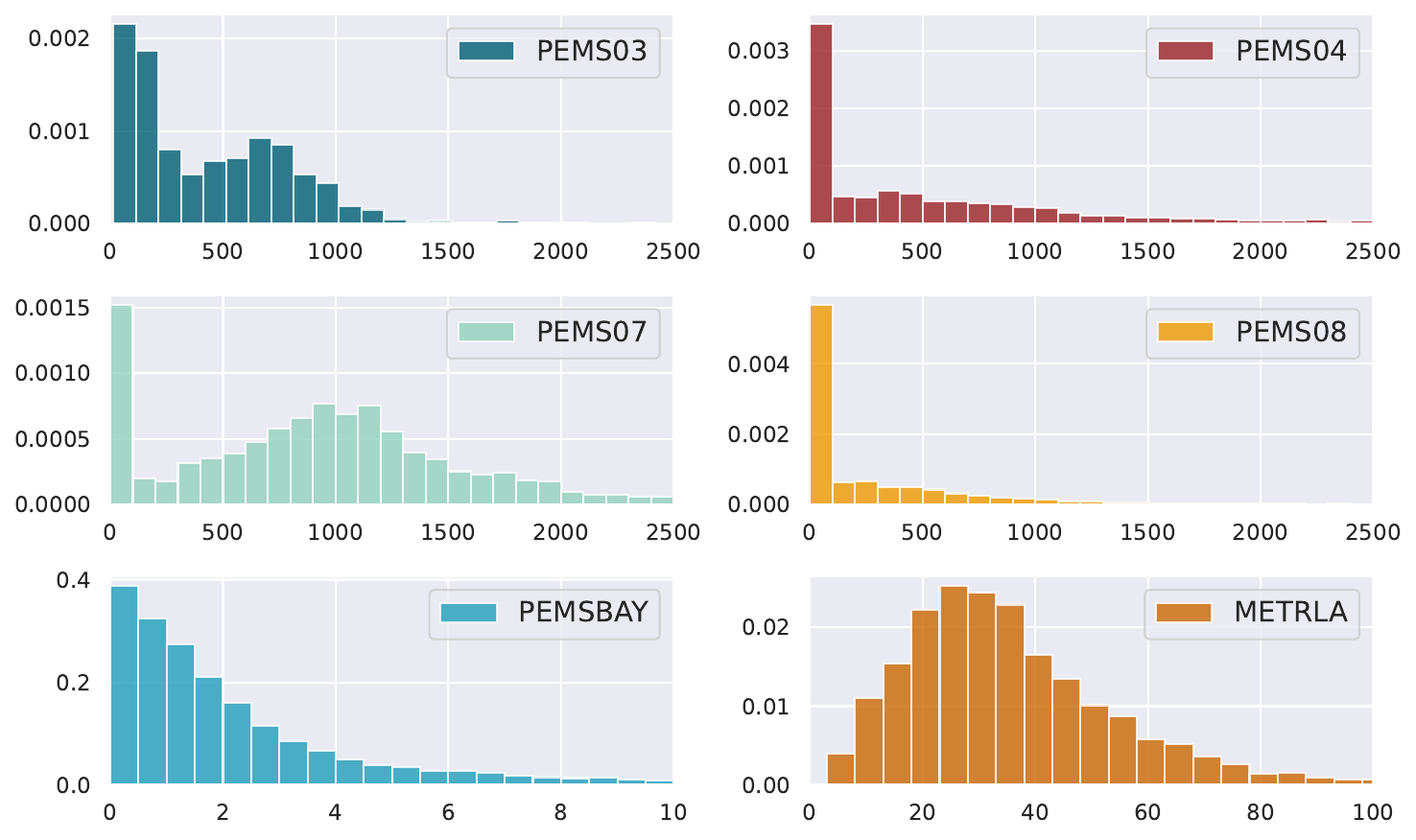}
    \caption{Normalized histogram of $\{ {\rm tr}(\bm{\Sigma}(\vs)): \forall \vs \in \gS_v, \forall v \in [n] \}$ for different datasets.}
    \label{fig:S_chars}
\end{figure}

\paragraph{Explaining $\Psi_{\texttt{T}}$}
Next, we discuss the rationale behind $\Psi_{\texttt{T}}$, which is designed to identify periodic patterns. For instance, in many traffic networks, trends exhibit weekly cycles, with distinct patterns on weekdays compared to weekends. Moreover, on an annual basis, the influence of holidays on traffic can be discerned, as people engage in shopping and other leisure activities. In Fig.~\ref{fig:mspaceT}, we have shown the traffic flow value of \texttt{PEMS04} with weekly (a) and daily (b) periodicity. Fpr the weekly periodic view (a), the trend is more pronounced with less deviation from the mean while for the daily view (b), a scattered trend is visible with high variance across states.
\begin{figure}[h!]
\centering
    \subfigure[\texttt{PEMS04}: Weekly]{\includegraphics[width=0.48\columnwidth]{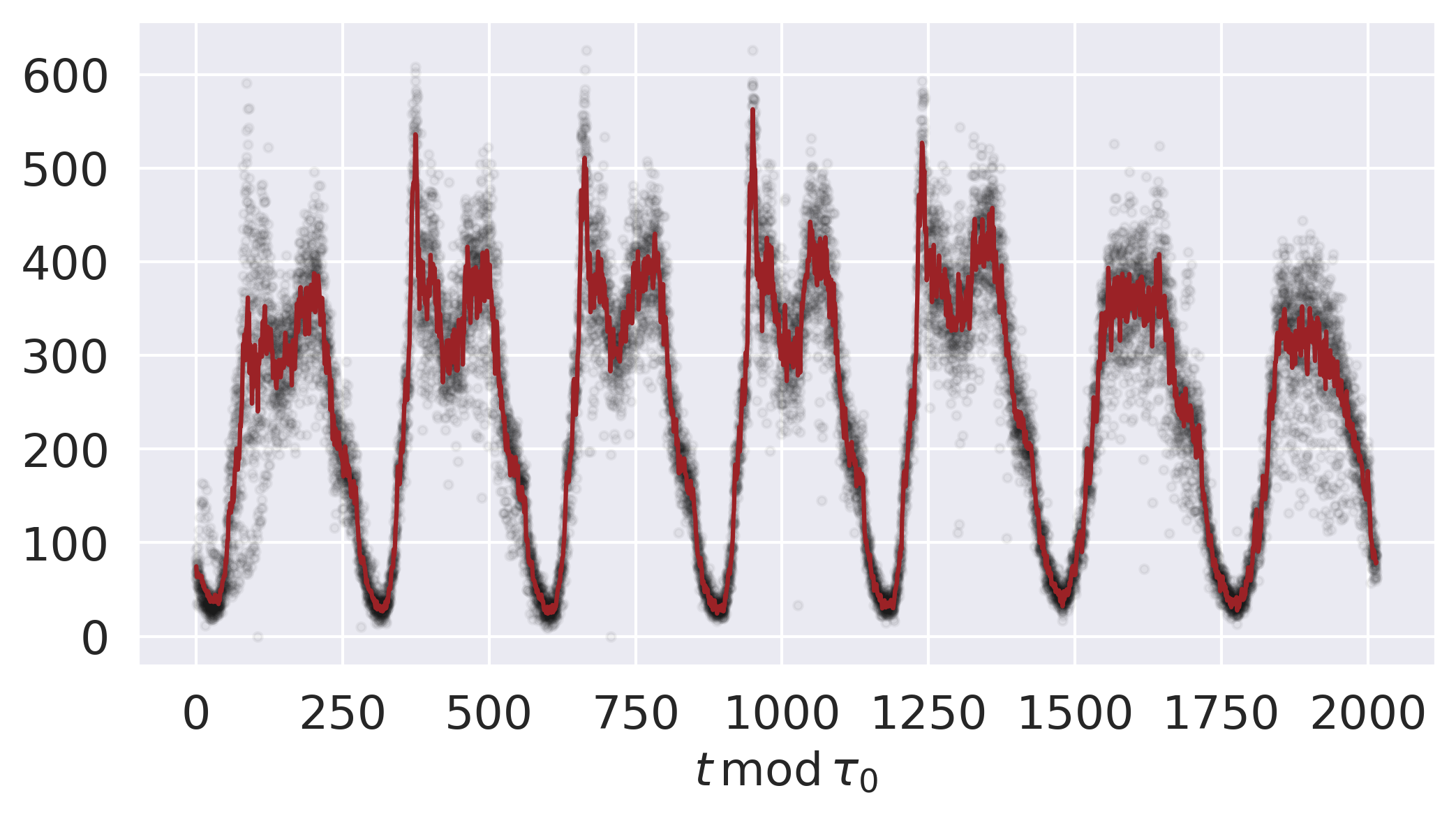}\label{fig:pems04_weekly}}
    \subfigure[\texttt{PEMS04}: Daily]{\includegraphics[width=0.48\columnwidth]{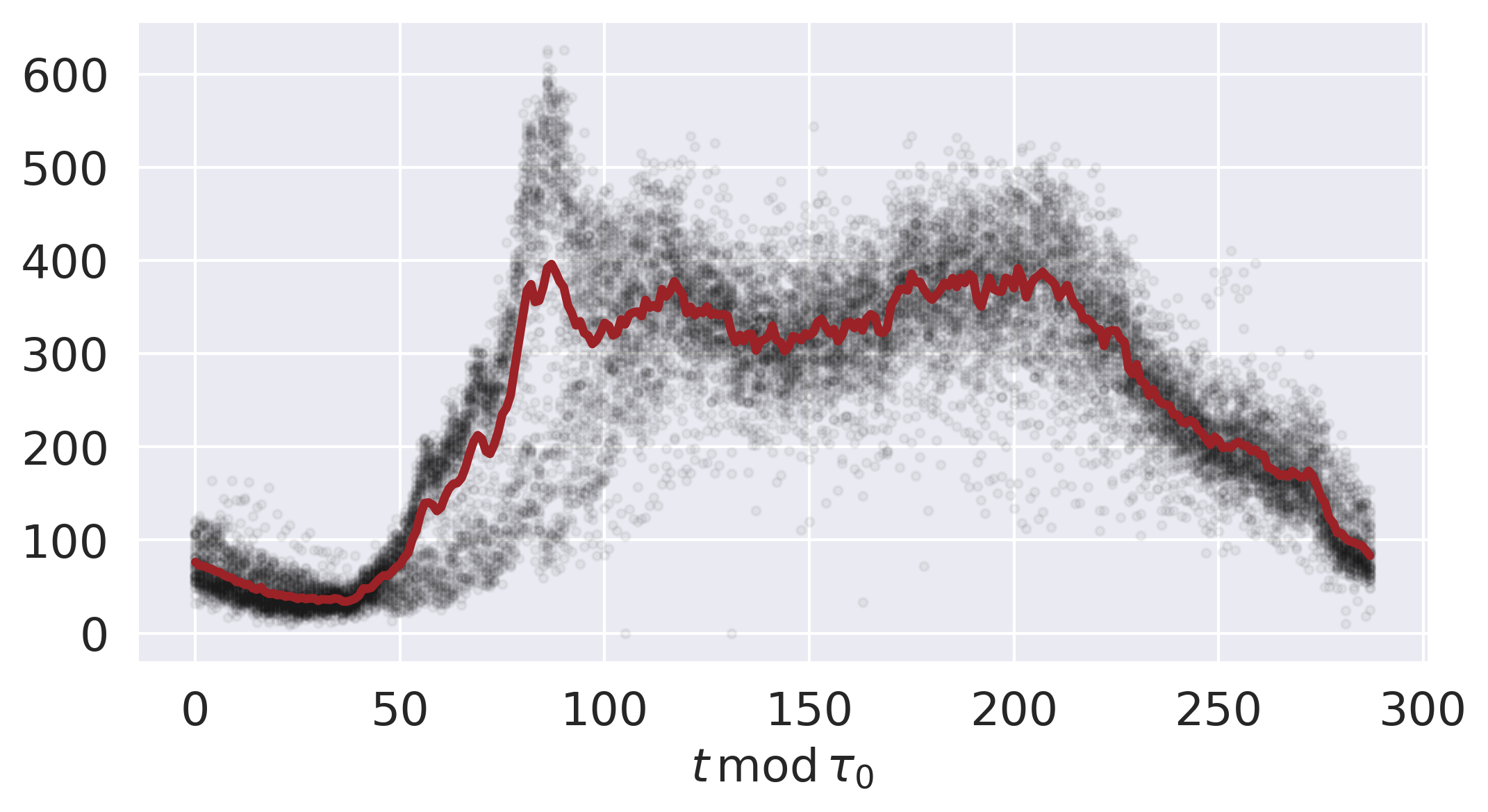}\label{fig:pems04_daily}}
    \caption{Periodic trends in the traffic dataset \texttt{PEMS04}; the black points represent the data-points, and the red line is the mean estimate for each state $t \,{\rm mod} \, \tau_0$.}\label{fig:mspaceT}
\end{figure}
\paragraph{Explaining the Contrastive Performance}
In this section, we look at the data of \texttt{PEMS04}, \texttt{PEMSBAY}, and \texttt{METRLA} and explain why \texttt{mspace-T$\mu$} outperforms the baselines on \texttt{PEMS04} but fails on \texttt{METRLA}. We also explain the relatively better performance of \texttt{mspace-T$\mu$} on \texttt{PEMSBAY} in contrast to \texttt{METRLA}.

First, we compare the weekly trend of \texttt{PEMS04} in Fig.~\ref{fig:mspaceT}(a) with the weekly trend of \texttt{METRLA} in Fig.~\ref{fig:contrast}(b). While the mean-trend for \texttt{PEMS04} is strong with little variance, for \texttt{METRLA}, we observe that the mean is lies in between two dense regions on data-points, far from each which suggests that the mean-trend has high variance which contributes to the failure of \texttt{mspace-T$\mu$}.

Similarly, in Fig.~\ref{fig:contrast}(a), we observe the mean-trend for \texttt{PEMSBAY} which has points of higher variance compared to \texttt{PEMS04} but it is not as bad as \texttt{METRLA}. This difference is also reflected in the performance of \texttt{PEMSBAY} which is at par with the baselines (see Table~\ref{result:multi_rmse_mae}).

\begin{figure}[h!]
\centering
    \subfigure[\texttt{PEMSBAY}: Weekly]{\includegraphics[width=0.48\columnwidth]{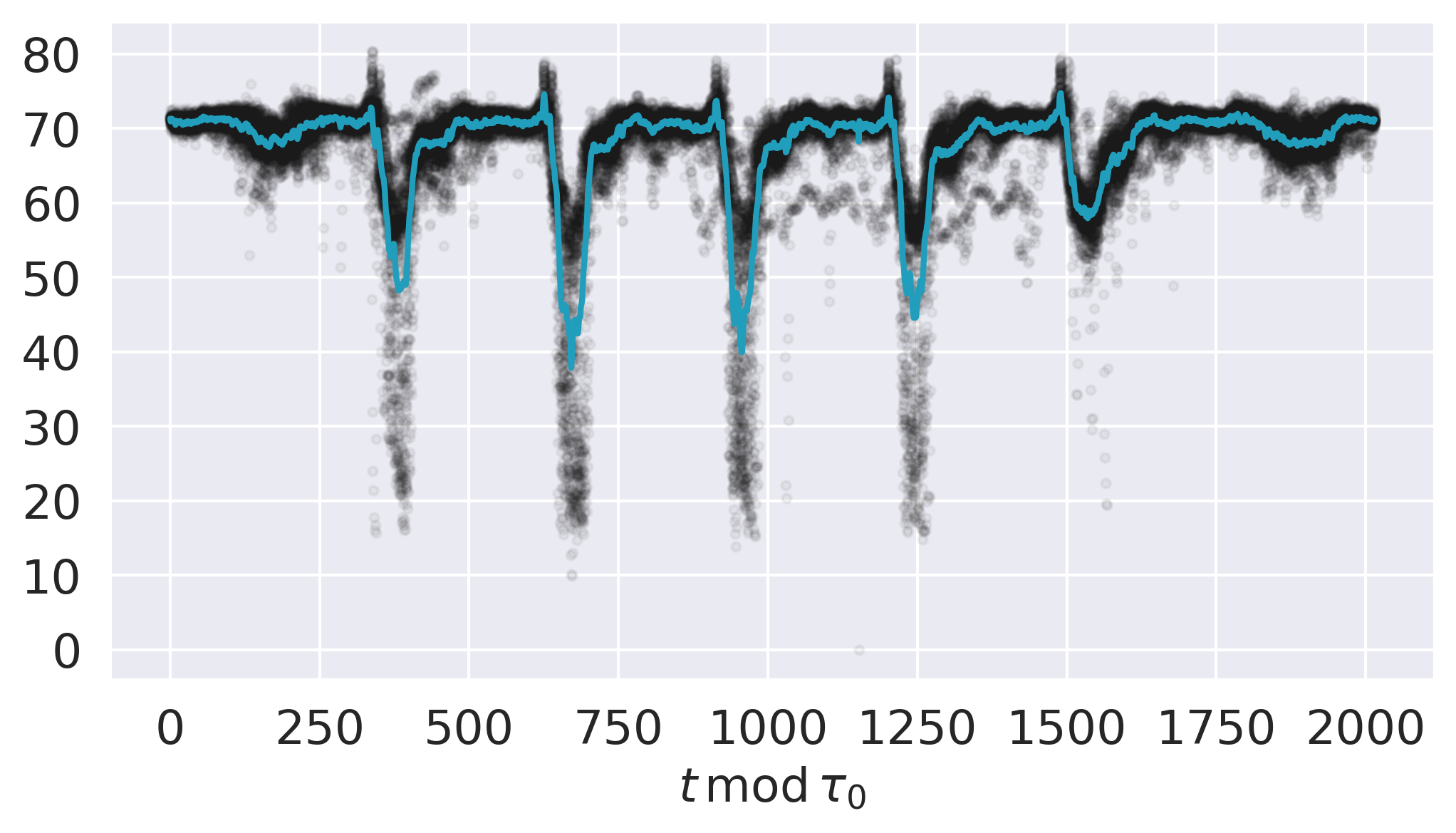}\label{fig:pems04_weekly}}
    \subfigure[\texttt{METRLA}: Weekly]{\includegraphics[width=0.48\columnwidth]{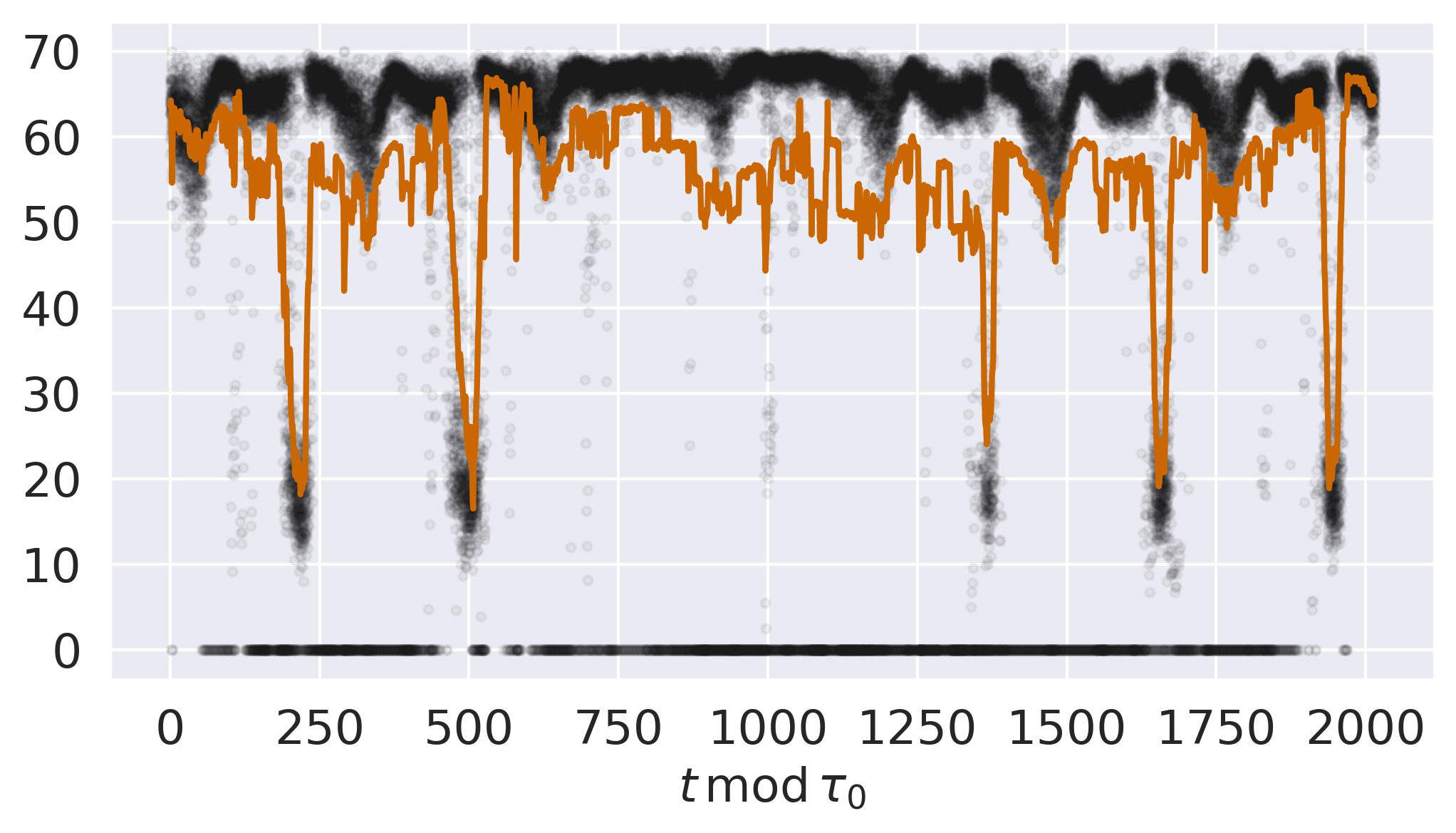}\label{fig:pems04_daily}}
    \caption{Periodic trends in the traffic datasets \texttt{PEMS04} and \texttt{METRLA}; the black points represent the data-points, and the coloured line is the mean estimate for each state $t \, {\rm mod} \, \tau_0$. The darker shades indicates higher density of data-points.}\label{fig:contrast}
\end{figure}

\section{Error Bounds}
\label{app:bound}
\paragraph{Upper Bound} We derive the upper bound on the RMSE for $q$-step iterative forecast below.
\begin{proof}[Proof of Theorem~\ref{thm:rmse}]

    For nodes in $\gU_v, v \in [n]$, the shock at time $t$ is sampled from a Gaussian distribution, the parameters of which depend on the previous shock $\hat{\vep}_{t-1}^{\langle \gU_v \rangle}$ through the state function:
    \begin{align}
        \hat{\vep}_t^{\langle \gU_v \rangle} \sim \gN \left( \hat{\vep}; \bm{\mu}\left( \Psi_{\texttt{S}}\left(  \hat{\vep}_{t-1}^{\langle \gU_v \rangle} \right) \right), \bm{\Sigma}\left(  \Psi_{\texttt{S}}\left(  \hat{\vep}_{t-1}^{\langle \gU_v \rangle} \right) \right) \right)        
    \end{align}
     
    We denote the shock estimated for node $v$ at time $t$ as:   
    \begin{align}
        \hat{\vep}_t(v) = \hat{\vep}_t^{\langle \gU_v \rangle}(v) \sim \gN \left( \hat{\vep}; \bm{\mu}_v\left( \Psi_{\texttt{S}}\left(  \hat{\vep}_{t-1}^{\langle \gU_v \rangle} \right) \right), \bm{\Sigma}_v\left(  \Psi_{\texttt{S}}\left(  \hat{\vep}_{t-1}^{\langle \gU_v \rangle} \right) \right) \right)
    \end{align}

    The mean square error for $q$-step iterative node feature forecasting is defined as:
    \begin{align}
        \textrm{MSE}(q) &\triangleq \frac{1}{ndq}\mathbb{E} \left[ \sum_{v \in [n]}  \sum_{i \in [q]} \lnorm \sum_{j \in [i]} \hat{\vep}_{t+j}(v) - \vep_{t+j}(v)  \rnorm^2 \right] \nonumber \\
        &=  \frac{1}{ndq} \sum_{v \in [n]}  \sum_{i \in [q]} \mathbb{E} \left[ \lnorm \sum_{j \in [i]} \hat{\vep}_{t+j}(v) - \vep_{t+j}(v) \rnorm^2 \right]. 
    \end{align}

    The shock difference between the true shock and predicted shock also follows a Gaussian distribution:
    \begin{align}
        \hat{\vep}_{t+j}(v) - \vep_{t+j}(v)  \sim  \gN \left( \vep;  \bm{\mu}_v\left( \Psi_{\texttt{S}}\left(  \hat{\vep}_{t_j-1}^{\langle \gU_v \rangle} \right) \right) - \vep_{t+j}(v), \bm{\Sigma}_v\left(  \Psi_{\texttt{S}}\left(  \hat{\vep}_{t+j-1}^{\langle \gU_v \rangle} \right) \right) \right).
    \end{align}
    Since, the sum of Gaussian r.v.s is also Gaussian, we have:
    \begin{align}
        \sum_{j \in [i]}  \hat{\vep}_{t+j}(v) - \vep_{t+j}(v) \sim\gN \left( \vep;  \sum_{j \in [i]}  \bm{\mu}_v\left( \Psi_{\texttt{S}}\left(  \hat{\vep}_{t_j-1}^{\langle \gU_v \rangle} \right) \right) - \vep_{t+j}(v), 
        \sum_{j \in [i]} \bm{\Sigma}_v\left(  \Psi_{\texttt{S}}\left(  \hat{\vep}_{t+j-1}^{\langle \gU_v \rangle} \right) \right) \right).
    \end{align}
    Moreover, for a Gaussian r.v. $\rvx \sim \gN(\rvx; \bm{\mu}, \bm{\Sigma})$, $\E\left[ \lnorm \vx \rnorm^2 \right] = \lnorm \bm{\mu} \rnorm^2 + {\rm tr}(\bm{\Sigma})$.
    \begin{align}
        \E\left[ \lnorm  \sum_{j \in [i]} \hat{\vep}_{t+j}(v) -  \vep_{t+j}(v) \rnorm^2 \right] &= \lnorm \sum_{j \in [i]}  \bm{\mu}_v\left( \Psi_{\texttt{S}}\left(  \hat{\vep}_{t+j-1}^{\langle \gU_v \rangle} \right) \right) - \vep_{t+j}(v) \rnorm^2 \nonumber \\
        &\quad  + \sum_{j \in [i]}{\rm tr} \left( \bm{\Sigma}_v\left(  \Psi_{\texttt{S}}\left(  \hat{\vep}_{t+j-1}^{\langle \gU_v \rangle} \right) \right) \right). 
    \end{align}
    \begin{align}
        \lnorm \sum_{j \in [i]}  \bm{\mu}_v\left( \Psi_{\texttt{S}}\left(  \hat{\vep}_{t+j-1}^{\langle \gU_v \rangle} \right) \right) - \vep_{t+j}(v) \rnorm &\leq \sum_{j \in [i]} \lnorm   \bm{\mu}_v\left( \Psi_{\texttt{S}}\left(  \hat{\vep}_{t+j-1}^{\langle \gU_v \rangle} \right) \right) - \vep_{t+j}(v) \rnorm \nonumber \\
        &\leq i \cdot \max_{j \in [i]} \lnorm   \bm{\mu}_v\left( \Psi_{\texttt{S}}\left(  \hat{\vep}_{t+j-1}^{\langle \gU_v \rangle} \right) \right) - \vep_{t+j}(v) \rnorm  \nonumber \\
        &\leq i \cdot \max_{t, j \in \mathbb{N}} \lnorm   \bm{\mu}_v\left( \Psi_{\texttt{S}}\left(  \hat{\vep}_{t+j-1}^{\langle \gU_v \rangle} \right) \right) - \vep_{t+j}(v) \rnorm  \nonumber \\
        &= i \cdot \sqrt{\alpha_{v,1}}.
    \end{align}
    \begin{align}
        \sum_{j \in [i]}{\rm tr} \left( \bm{\Sigma}_v\left(  \Psi_{\texttt{S}}\left(  \hat{\vep}_{t+j-1}^{\langle \gU_v \rangle} \right) \right) \right) &\leq i \cdot \max_{j \in [i]}{\rm tr} \left( \bm{\Sigma}_v\left(  \Psi_{\texttt{S}}\left(  \hat{\vep}_{t+j-1}^{\langle \gU_v \rangle} \right) \right) \right) \leq i \cdot \alpha_{v,2}.
    \end{align}

    \begin{align}
        \E\left[ \lnorm  \sum_{j \in [i]} \hat{\vep}_{t+j}(v) - \vep_{t+j}(v)   \rnorm^2 \right] \leq \alpha_{v,1} \cdot i^2 + \alpha_{v,2}  \cdot i, \quad \alpha_{v,1}, \alpha_{v,2} \in \mathbb{R}^+.
    \end{align}

    \begin{align}
        {\rm MSE}(q) &\leq \frac{1}{ndq} \sum_{v \in [n]} \sum_{i \in [q]}   \alpha_{v,1} \cdot i^2 +  \alpha_{v,2} \cdot i \nonumber \\
        & = \frac{\sum_{v\in[n]} \alpha_{v,1}}{6nd} (q+1)(q+2) + \frac{\sum_{v\in[n]} \alpha_{v,2} }{2nd} (q+1).
    \end{align}

    Let $\alpha \triangleq  \frac{1}{6nd}\sum_{v\in[n]} \alpha_{v,1}$, and $\beta \triangleq \frac{1 }{2nd}\sum_{v\in[n]} \alpha_{v,2} $, then
    \begin{align}
        {\rm MSE}(q) \leq \alpha q^2 + (3\alpha + \beta)q + \beta.
    \end{align}
        
    By Jensen's inequality,
    \begin{align}
        {\rm RMSE}(q) \leq \sqrt{ {\rm MSE}(q) } \leq \sqrt{ \alpha q^2 + (3\alpha + \beta)q + \beta }.
    \end{align}
\end{proof}

The above proof is for \texttt{mspace-S$\gN$} and also applies to \texttt{mspace-T$\gN$}. For \texttt{mspace-S$\mu$} and \texttt{mspace-T$\mu$}, $\beta = 0$.

\paragraph{Lower Bound} Similarly, we can find a lower bound on the MSE for $q$-step iterative forecast:
\begin{align}
    \E\left[ \lnorm  \sum_{j \in [i]} \hat{\vep}_{t+j}(v) -  \vep_{t+j}(v) \rnorm^2 \right] &\geq \sum_{j \in [i]} {\rm tr} \left( \bm{\Sigma}_v\left(  \Psi_{\texttt{S}}\left(  \hat{\vep}_{t+j-1}^{\langle \gU_v \rangle} \right) \right) \right) \nonumber\\
    &\geq i \cdot \min_{j \in [i]} {\rm tr} \left( \bm{\Sigma}_v\left(  \Psi_{\texttt{S}}\left(  \hat{\vep}_{t+j-1}^{\langle \gU_v \rangle} \right) \right) \right) = i \cdot \alpha_{v,3}.
\end{align}
\begin{align}
    {\rm MSE}(q) &\geq \frac{1}{ndq} \sum_{v \in [n]} \sum_{i \in [q]} i \cdot \alpha_{v,3} = \Big( \underbrace{\frac{1}{nd}\sum_{v \in [n]} \alpha_{v,3}}_{\triangleq \beta'} \Big) \cdot (q+1) = \beta' q + \beta'.
\end{align}

\newpage

\section{Complexity Analysis}
\label{app:complexity}
\subsection{Computational Complexity}
\label{app:complex_comp}
We denote the computational complexity operator as $\mathfrak{C}(\cdot)$, the argument of which is an algorithm or part of an algorithm.

\begin{algorithm}[h!]
   \caption{\texttt{mspace-S$\gN$}}
   \label{alg:mspace_2}
\begin{algorithmic}[1]
   \INPUT $\gG = (\gV, \gE, \mX)$, $r\in [0,1)$, $q,M$
   \OUTPUT $\hat{\vep}_t(v), \quad \forall v \in \gV, t \in [\lfloor r \cdot T \rfloor, T]$
   \STATE $\vep_t(v) \leftarrow \vx_t(v) - \vx_{t-1}(v), \quad \forall v \in \gV, \, t \in [T]$ 

   \textit{Offline training $(\mathsf{A})$:}
   \FOR{$t \in [ \lfloor r \cdot T \rfloor ]$} 
   \FOR{$v \in \gV$}
   \STATE $\displaystyle \vs_t^{\langle \gU_v \rangle} \leftarrow \Psi\left( \vep_t^{\langle \gU_v \rangle}  \right)$ \COMMENT{$\sum_{v \in \gV}   d|\gU_v|$}
   \STATE $\gS_v \leftarrow \gS_v \cup \left\{ \vs_t^{\langle \gU_v \rangle}  \right\}$ \COMMENT{$n$}
   \STATE $\gQ_v\left(\vs_t^{\langle \gU_v \rangle} \right) \leftarrow$ enqueue $ \vep_{t+1}^{\langle \gU_v \rangle}$ \COMMENT{$n$}
   \ENDFOR
   \ENDFOR
   \STATE   $\bm{\mu}_v(\vs) \leftarrow {\rm mean}(\gQ_v(\vs)), \quad \forall s \in \gS_v, v \in \gV$ \COMMENT{$\sum_{v \in \gV} d|\gU_v||\gS_v|M$}
   \STATE   $\bm{\Sigma}_v(\vs) \leftarrow {\rm covariance}(\gQ_v(\vs)), \quad \forall s \in \gS_v, v \in \gV$ \COMMENT{$\sum_{v \in \gV} (d|\gU_v|)^2|\gS_v|M$}

   \textit{Online learning $(\mathsf{B})$:}
   \FOR{$t \in [ \lfloor r \cdot T \rfloor , T-q ] $}
   \FOR{$v \in \gV$}
   \STATE $\displaystyle \vs_t^{\langle \gU_v \rangle} \leftarrow \Psi\left( \vep_t^{\langle \gU_v \rangle}  \right)$ \COMMENT{$\sum_{v \in \gV}   d|\gU_v|$}
   \STATE $\displaystyle \vs^* \leftarrow \arg \min_{\vs \in \gS_v} \lnorm \vs - \vs_t^{\langle \gU_v \rangle} \rnorm$ \COMMENT{$\sum_{v \in \gV} d|\gU_v||\gS_v|$}
   \STATE $\hat{\vep}_{t+1}^{\langle \gU_v \rangle} \sim \gN( \vep; \bm{\mu}_v(\vs^*), \bm{\Sigma}_v(\vs^*)  )$ \COMMENT{$\sum_{v \in \gV} (|\gU_v|d)^3$}
   \FOR{$k \in [2, q]$}
   \STATE $\displaystyle \vs^* \leftarrow \arg \min_{\vs \in \gS_v} \lnorm \vs - \Psi\left( \hat{\vep}_{t+k-1}^{\langle \gU_v \rangle}  \right) \rnorm$ \COMMENT{$(q-1) \times \sum_{v \in \gV} d|\gU_v|(1+|\gS_v|)$}
   \STATE $\hat{\vep}_{t+k}^{\langle \gU_v \rangle} \sim \gN( \vep; \bm{\mu}_v(\vs^*), \bm{\Sigma}_v(\vs^*)  )$ \COMMENT{$(q-1) \times \sum_{v \in \gV} (|\gU_v|d)^3$}
   \ENDFOR
   \STATE $\hat{\vep}_{t+k}(v) \leftarrow \hat{\vep}_{t+k}^{\langle \gU_v \rangle}(v), \quad \forall k \in [q]$
   \STATE Update $\gS_v, \gQ_v$ \COMMENT{$2n$}
   \STATE Update $\bm{\mu}_v(\vs), \bm{\Sigma}_v(\vs), \quad \forall \vs \in \gS_v$ \COMMENT{$\sum_{v \in \gV} (d|\gU_v| + d^2|\gU_v|^2)|\gS_v|M$}
   \ENDFOR
   \ENDFOR
\end{algorithmic}
\end{algorithm}

Computational complexity of offline training for \texttt{mspace-S$\gN$} can be written as:
\begin{align}
    \mathfrak{C}(\mathsf{A}) = \gO\Bigg( \underbrace{\lfloor rT \rfloor d \sum_v  |\gU_v|}_{[4]} +   \underbrace{\lfloor rT \rfloor 2n}_{[5], [6]} +  \underbrace{d M  \sum_v |\gU_v||\gS_v|}_{[9] ({\rm mean})} +  \underbrace{d^2 M \sum_v  |\gU_v|^2 |\gS_v|}_{[10] ({\rm covariance})}   \Bigg).
    \label{offline}
\end{align}

Computational complexity of online learning for \texttt{mspace-S$\gN$} can be written as:
\begin{align}
     \mathfrak{C}(\mathsf{B}) =\gO\Bigg( \sum_{t= \lceil rT \rceil}^{T-q} \Bigg\{  \underbrace{dq \sum_v |\gU_v|}_{[13], [17]}+  \underbrace{dq \sum_v |\gU_v||\gS_v|}_{[14], [17]} +  \underbrace{d^3q \sum_v |\gU_v|^3}_{[15], [18] ({\rm sampling})} +  \underbrace{2n}_{[21]} \nonumber \\
     +  \underbrace{d M  \sum_v |\gU_v| |\gS_v|}_{[22] ({\rm mean})} +  \underbrace{d^2 M \sum_v  |\gU_v|^2 |\gS_v| }_{[22] ({\rm covariance})} \Bigg\}  \Bigg).
    \label{online}
\end{align}

\begin{lemma}
    The computational complexity of \texttt{mspace-S$\gN$} is:
    \begin{align*}
        \mathfrak{C}(\mathsf{A}) &= \gO \Big( dbn rT + d^2 b^2 Mn \cdot \min\{rT, 2^{bd}\} \Big),\\
        \mathfrak{C}(\mathsf{B}) &= \gO \left( (1-r)Tnd^2b^2 \left( qdb +  M \cdot \min\left\{ \frac{(1+r)}{2}T, 2^{bd}  \right\} \right)  \right),
    \end{align*}
    where $b = \max_{v \in [n]} |\gU_v|$.
    \label{lm:SN}
\end{lemma}
\begin{proof}
    We denote the maximum degree of a node as $b \triangleq \max_{v \in [n]} |\gU_v| <n $ which does not necessarily scale with $n$ unless specified by the graph definition. Furthermore, the total number of states observed for a node till time step $t \in \mathbb{N}$ cannot exceed $t$, i.e., $|\gS_v| \leq t$. We also know the total number of states theoretically possible for node $v$ is $2^{|\gU_v|d}$ for $\Psi_{\texttt{S}}(\cdot)$. Therefore, the number of states observed till time $t$ for node $v$ is upper bounded as: $|\gS_v| \leq \min\left\{ t, 2^{bd} \right\}$.
    Based on this, we can simplify \eqref{offline}, and \eqref{online} as follows:
    \begin{align*}
    \mathfrak{C}(\mathsf{A}) &= \gO \left(  dbnrT + 2nrT + (dbM + d^2 b^2 M) \cdot n\min\{rT, 2^{bd}\} \right)\\
     &= \gO \Big( dbn rT + d^2 b^2 Mn \cdot \min\{rT, 2^{bd}\} \Big).\\
     \mathfrak{C}(\mathsf{B}) &= \gO \left( \sum_{t= \lceil rT \rceil}^{T-q} qdbn + qd^3b^3n + 2n + db(q+M)n \cdot \min\{t, 2^{bd}\} + d^2b^2Mn \cdot \min\{t, 2^{bd}\} \right)\\
    &= \gO \left( \sum_{t= \lceil rT \rceil}^{T-q}  qd^3b^3 n+ (db(q+M) + d^2b^2M )n\cdot \min\{t, 2^{bd}\} \right)\\
     &= \gO \Big( (1-r)T \cdot qd^3b^3n +    d^2b^2Mn \cdot \min\{ (1-r^2)T^2, 2^{bd}(1-r)T  \}   \Big)\\
     &= \gO \left( (1-r)Tn \left( qd^3b^3 +  d^2b^2M \cdot \min\left\{ \frac{(1+r)}{2}T, 2^{bd}  \right\} \right)  \right).
    \end{align*}
\end{proof} 

\begin{lemma}
    The computational complexity of \texttt{mspace-S$\mu$} is:
    \begin{align*}
        \mathfrak{C}(\mathsf{A}) &= \gO\Big( dbn rT + dbMn \cdot \min\{rT, 2^{bd}\} \Big),\\
        \mathfrak{C}(\mathsf{B}) &= \gO \left( (1-r)Tndb (q+M)\cdot \min \left\{ \frac{(1+r)}{2}T, 2^{bd}  \right\}  \right).
    \end{align*}
    \label{lm:Smu}
\end{lemma}
\begin{proof}
    The sampling steps [15], and [18] in Algorithm~\ref{alg:mspace_2} are replaced with $\hat{\vep}_{t}^{\langle \gU_v \rangle} \leftarrow \bm{\mu}(\vs^*)$ which has a computational complexity of $\gO(d|\gU_v|)$. Moreover, $\Omega_{\mu}(\cdot)$ does not require the covariance matrix, therefore we do not need to compute it. We simplify the computational complexity expressions as:
    \begin{align*}
    \mathfrak{C}(\mathsf{A}) &= \gO\left( \lfloor rT \rfloor d \sum_v  |\gU_v| +   \lfloor rT \rfloor 2n+  d M  \sum_v |\gU_v||\gS_v|  \right)\\ 
     &=    \gO\Big( dbn rT + dbMn \cdot \min\{rT, 2^{bd}\} \Big).\\
    \mathfrak{C}(\mathsf{B}) &= \gO\Bigg( \sum_{t= \lceil rT \rceil}^{T-q} \Big\{  dq \sum_v |\gU_v|+  dq \sum_v |\gU_v||\gS_v|+  \underbrace{dq \sum_v |\gU_v|}_{({\rm sampling})} +  2n +  d M  \sum_v |\gU_v| |\gS_v| \Big\}  \Bigg)\\
    &=\gO \left( \sum_{t= \lceil rT \rceil}^{T-q} 2qdbn  + 2n + db(q+M)n\cdot \min\{t, 2^{bd}\} \right) \\
    &= \gO \left( (1-r)Tndb (q+M)\cdot \min \left\{ \frac{(1+r)}{2}T, 2^{bd}  \right\}  \right).
    \end{align*}
\end{proof}

\begin{lemma}
    The computational complexity of \texttt{mspace-T$\gN$} is:
    \begin{align*}
        \mathfrak{C}(\mathsf{A}) &= \gO\left( nrT + d^2Mn \tau_0   \right),\\
        \mathfrak{C}(\mathsf{B}) &= \gO\Big( (1-r)Tnd^2 \cdot (  M\tau_0 + qd ) \Big).
    \end{align*}
    \label{lm:TN}
\end{lemma}
\begin{proof}
    For the state function $\Psi_{\texttt{T}}$, the total number of states for any node is the period $\tau_0 \in \mathbb{N}$, i.e., $|\gS_v| \leq \tau_0$. Moreover, the state calculation $\vs_t \leftarrow \Psi(t)$ has computational complexity of $\gO(1)$. Most importantly, for $\Psi_{\texttt{T}}$, $b = 1$ as it only focuses on the seasonal trends.
    \begin{align*}
    \mathfrak{C}(\mathsf{A}) &= \gO\left( \lfloor rT \rfloor  \sum_v  1 +  \lfloor rT \rfloor 2n +  d M  \sum_v |\gU_v||\gS_v|  +  d^2 M \sum_v  |\gU_v|^2|\gS_v|   \right)\\
    &= \gO\left( 3nrT + dMn\tau_0 + d^2Mn \tau_0   \right) = \gO\left( nrT + d^2Mn \tau_0   \right).\\
    \mathfrak{C}(\mathsf{B}) &= \gO\Bigg( \sum_{t= \lceil rT \rceil}^{T-q} \Big\{  q \sum_v 1 + dq \sum_v |\gU_v||\gS_v| +  d^3q \sum_v |\gU_v|^3 +  2n \\ 
    &\qquad \qquad \qquad + d M  \sum_v |\gU_v||\gS_v| +  d^2 M \sum_v  |\gU_v|^2|\gS_v| \Big\}  \Bigg)\\
    &=  \gO\left(  \{ q + dq  \tau_0 +  q d^3  +  2 + d M  \tau_0 +  d^2 M  \tau_0 \} \cdot n(1-r)T   \right) \\
    &= \gO\Big( (1-r)Tnd^2  \cdot (  M\tau_0 + qd ) \Big).
\end{align*}
\end{proof}

\begin{lemma}
    The computational complexity of \texttt{mspace-T$\mu$} is:
    \begin{align*}
        \mathfrak{C}(\mathsf{A}) &= \gO\left( nrT + dMn \tau_0   \right),\\
        \mathfrak{C}(\mathsf{B}) &= \gO\Big( (1-r)Tn \cdot d (q+M)\tau_0 \Big).
    \end{align*}
    \label{lm:Tmu}
\end{lemma}
\begin{proof} Based on the explanation provided for \texttt{mspace-T$\gN$}, we simplify the computational complexity expressions for \texttt{mspace-T$\mu$} as:
    \begin{align*}
    \mathfrak{C}(\mathsf{A}) &= \gO\left( \lfloor rT \rfloor  \sum_v  1 +  \lfloor rT \rfloor 2n +  d M  \sum_v |\gS_v|   \right)\\
    &= \gO\left( 3nrT + dMn\tau_0    \right) = \gO\left( nrT + dMn \tau_0   \right).\\
    \mathfrak{C}(\mathsf{B}) &= \gO\left( \sum_{t= \lceil rT \rceil}^{T-q} \left\{  q \sum_v 1 + dq \sum_v |\gS_v| +  2n + d M  \sum_v |\gS_v|  \right\}  \right)\\
    &=  \gO\Big(  \{ q + dq  \tau_0  +  2 + d M  \tau_0  \} \cdot n(1-r)T   \Big) = \gO\Big( (1-r)Tn \cdot d (q+M)\tau_0 \Big).
\end{align*}
\end{proof}

\subsection{Space Complexity}
\label{app:space}

We denote the space complexity operator as $\mathfrak{M}(\cdot)$, the argument of which is an algorithm or part of an algorithm. The variables in  offline training $\mathsf{A}$ are re-used in online learning $\mathsf{B}$. Therefore, we can say that $\mathfrak{M}(\mathsf{B}) = \mathfrak{M}(\mathsf{A} \cup \mathsf{B})$.

In an implementation of \texttt{mspace} where forecasting is sequentially performed for each node $v \in [n]$, memory space can be efficiently reused, except for storing the outputs. This approach optimises memory usage, resulting in a space complexity characterised by:
\begin{align}
    \mathfrak{M}(\mathsf{A} \cup \mathsf{B}) = \gO \left( \max_{\substack{v \in [n],\\ t \in [T]}}    \underbrace{d|\gU_v||\gS_v|}_{\gS_v} + \underbrace{cM d|\gU_v||\gS_v|}_{\gQ_v( \vs) \, \forall \vs \in \gS_v } + \underbrace{cd|\gU_v||\gS_v|}_{\bm{\mu}_v(\vs) \, \forall \vs \in \gS_v} + \underbrace{c(d|\gU_v|)^2|\gS_v|}_{\bm{\Sigma}_v(\vs) \, \forall \vs \in \gS_v}  + \underbrace{d|\gU_v|}_{\vs^*}  \right).
    \label{eq:space}
\end{align}
\begin{lemma}
    The space complexity of \texttt{mspace-S$\gN$} is $\mathfrak{M}( \mathsf{A} \cup \mathsf{B}) = \gO \Big( db(M + db) \cdot  \min\{ T, 2^{bd} \}    \Big).$
    \label{lm:space_SN}
\end{lemma}
\begin{proof} Simplifying \eqref{eq:space} results in:
    \begin{align*}
    \mathfrak{M}( \mathsf{A} \cup \mathsf{B}) &= \gO \left( \max_{\substack{v \in [n],\\ t \in [T]}} \, (db + cMdb + cdb + cd^2b^2) |\gS_v| + db \right)\\
    &= \gO \Big( (cMdb + cd^2b^2) \cdot \max_{t \in [T]} \min\{ t, 2^{bd} \}    \Big) =  \gO \Big( db(M + db) \cdot  \min\{ T, 2^{bd} \}    \Big).
    \end{align*}
\end{proof}

\begin{lemma}
    The space complexity of \texttt{mspace-S$\mu$} is $\mathfrak{M}( \mathsf{A} \cup \mathsf{B}) = \gO \Big( Mdb  \cdot  \min\{ T, 2^{bd} \}    \Big).$
    \label{lm:space_Smu}
\end{lemma}
\begin{proof}
    Some space is saved in \texttt{mspace-S$\mu$}, as we do not need to store the covariance matrices.
    \begin{align*}
    \mathfrak{M}( \mathsf{A} \cup \mathsf{B}) = \gO \left( \max_{\substack{v \in [n],\\ t \in [T]}} \, (db + cMdb + cdb) |\gS_v| + db  \right) =  \gO \Big( Mdb  \cdot  \min\{ T, 2^{bd} \}    \Big).
    \end{align*}
\end{proof}

\begin{lemma}
    The space complexity of \texttt{mspace-T$\gN$} is $\mathfrak{M}( \mathsf{A} \cup \mathsf{B}) =  \gO \Big( d(M + d) \tau_0    \Big).$
    \label{lm:space_TN}
\end{lemma}
\begin{proof}
    As explained earlier, for the state function $\Psi_{\texttt{T}}$, $b=1$. Therefore, the queues only store the shock vectors for a single node, and not the neighbours. The space complexity expression is simplified as:
    \begin{align*}
    \mathfrak{M}( \mathsf{A} \cup \mathsf{B}) = \gO \left( \max_{\substack{v \in [n],\\ t \in [T]}} \, (d + cMd + cd + cd^2) |\gS_v| + db \right) = \gO \Big( d(M + d) \tau_0    \Big).
    \end{align*}
\end{proof}

\begin{lemma}
    The space complexity of \texttt{mspace-T$\mu$} is $\gO \Big( Md \tau_0    \Big).$
    \label{lm:space_Tmu}
\end{lemma}
\begin{proof}
    $\mathfrak{M}( \mathsf{A} \cup \mathsf{B}) = \gO \left( \max_{\substack{v \in [n],\\ t \in [T]}} \, (d + cMd + cd) |\gS_v| + d  \right) = \gO \Big( Md \tau_0    \Big).$
\end{proof}

\subsection{Asymptotic Analysis} 
\label{app:asymp_complex}
Theorem~\ref{thm:complex} states that \textit{for asymptotically large number of nodes $n$ and timesteps $T$, the computational complexity of  \texttt{mspace} is $\gO(nT)$, and the space complexity is $\gO(1)$ across all variants}.
\begin{proof}
 We analyse the lemmas \ref{lm:SN}-\ref{lm:space_Tmu} introduced in this section for the asymptotic case of very large $n$ and $T$.
For very large $T$, $\min\left\{ \frac{(1+r)}{2}T, 2^{bd}  \right\} \rightarrow 2^{bd}$. Similarly, $ \min\{ T, 2^{bd} \} \rightarrow 2^{bd}$. Considering the terms $r, d, M, q, \tau_0, b$ as constants, the computational complexity for both offline and online parts of all the \texttt{mspace} variants becomes $\gO(nT)$ for asymptotically large $n,T$.

Furthermore, the space complexity terms lack $n$ or $T$ for very large $T$, which allows us to conclude that the space complexity of all the variants of \texttt{mspace} is constant, i.e., $\gO(1)$.
\end{proof}


\newpage
\section{Experiments on Synthetic Dataset}
\label{app:synthetic}
We generate datasets through Algorithm~\ref{alg:syn_data} by supplying the parameters outlined in Table~\ref{tab:syn_param}. For each dataset, we create multiple random instances and report the mean and standard deviation of the metrics in the results.
\begin{table}[htbp]
    \centering
    \small 
    \caption{Parameters for different synthetic dataset packages.}
    \label{tab:syn_param}
    \resizebox{\columnwidth}{!}{%
    \begin{tabular}{lcccccccccccc} \toprule
        Dataset & $\gG \sim$ & $d$ & $T$ & $\mu_{\min}$ & $\mu_{\max}$ & $\sigma_{\min}$ & $\sigma_{\max}$ & $\mu_0$ & $\sigma_0$ & $\tau$ & $\mu_{\tau}$ & $\sigma_{\tau}$ \\ \midrule
        \texttt{SYN01} & $\mathfrak{G}_{\rm ER}\left(20,0.2\right)$ & 1 & $10^3$ & $-200$ & $200$ & $40$ & $50$ & $2\times10^4$ & $5000$ & $100$ & $100$ & $20$ \\
        \texttt{SYN02} & $\mathfrak{G}_{\rm ER}\left(20, 0.2\right)$ & 1 & $10^3$ & $-200$ & $200$ & $40$ & $50$ & $2\times10^4$ & $5000$ & 0 & &\\
        \texttt{SYN03} & $\mathfrak{G}_{\rm ER}\left(40, 0.5\right)$ & 1 & $10^3$ & $-400$ & $400$ & $30$ & $40$ & $10^4$ & $2000$ & 0 & &\\
        \texttt{SYN04} & $\mathfrak{G}_{\rm ER}\left(40, 0.5\right)$ & 1 & $10^4$ & $-400$ & $400$ & $30$ & $40$ & $10^4$ & $2000$ & 0 & &\\ \bottomrule
    \end{tabular}}
\end{table}

\begin{figure}[h!]
\centering
    \subfigure[\texttt{SYN01}]{\includegraphics[width=0.48\columnwidth]{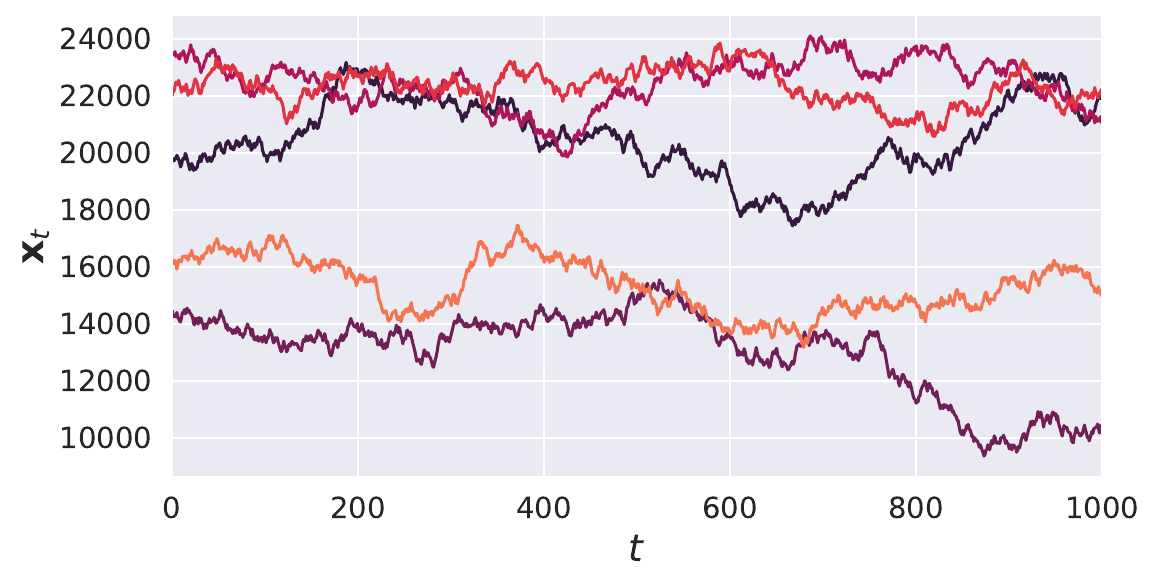}\label{fig:pems04_weekly}}
    \subfigure[\texttt{SYN02}]{\includegraphics[width=0.48\columnwidth]{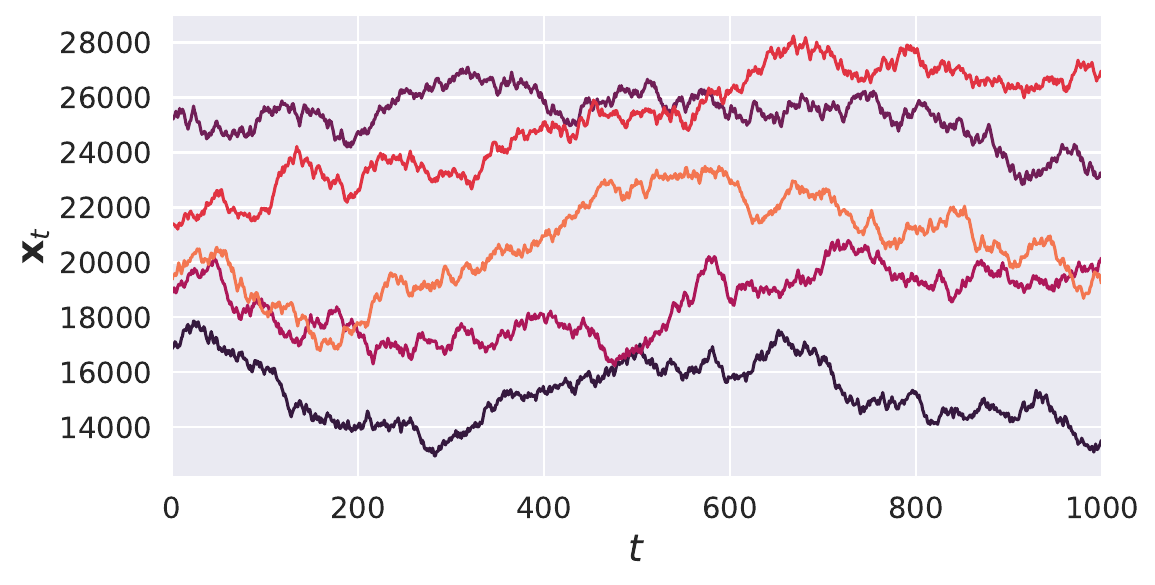}\label{fig:pems04_daily}}
    \subfigure[\texttt{SYN03}]{\includegraphics[width=0.48\columnwidth]{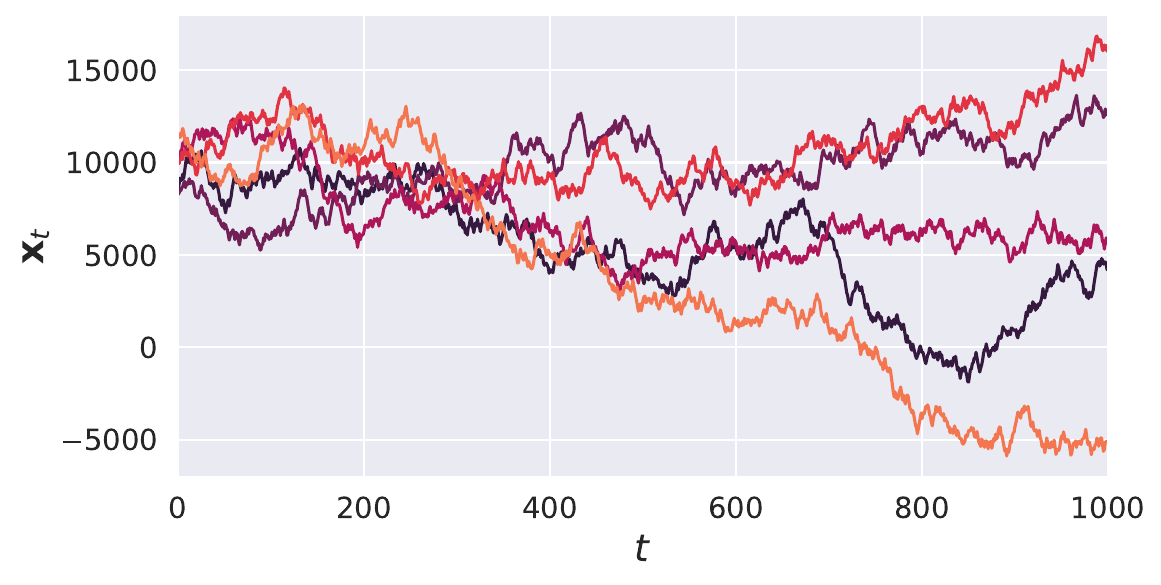}\label{fig:pems04_weekly}}
    \subfigure[\texttt{SYN04}]{\includegraphics[width=0.48\columnwidth]{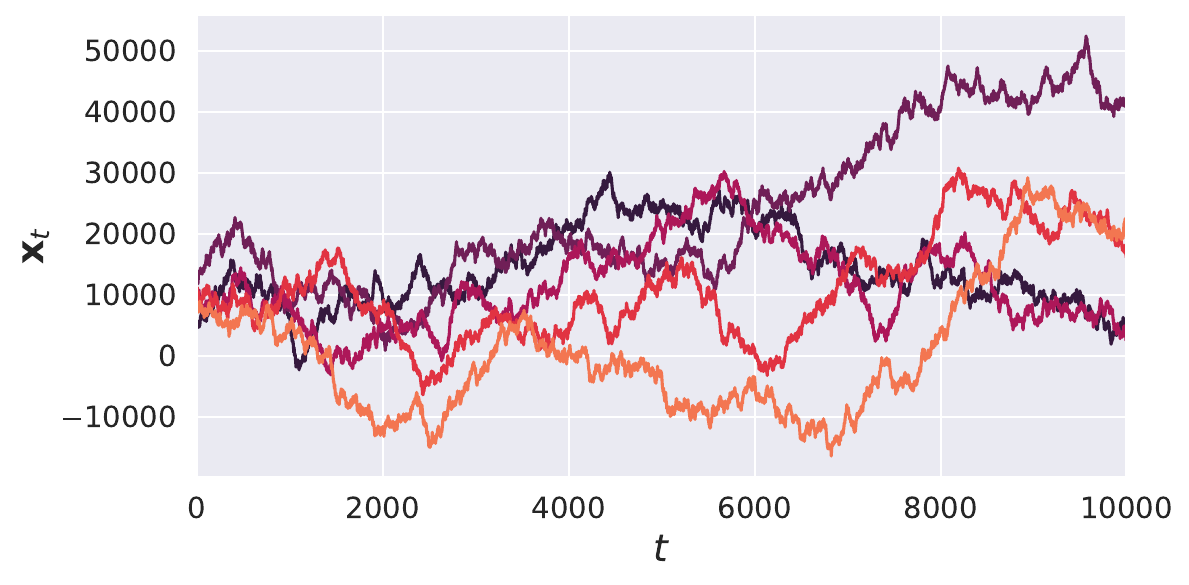}\label{fig:pems04_daily}}
    \caption{Exemplary synthetic dataset samples shown for 5 nodes.}\label{fig:syn_example}
\end{figure}


\paragraph{Periodicity}
The generator parameters for \texttt{SYN01} and \texttt{SYN02} are same except for the periodic component added to \texttt{SYN01} which has a period of $\tau=100$ timesteps consisting of shocks sampled from $\gN(100,20)$. An algorithm which can exploit the periodic influence in the signal should perform better on \texttt{SYN01} compared to \texttt{SYN02}. The models which perform worse on periodic dataset are marked \textcolor{red}{red}.
\begin{table}[h!]
    \centering
    \small 
    \caption{Impact of data periodicity on RMSE achieved by different models.}
    \begin{tabular}{lrclcrclc}
        \toprule
        & \multicolumn{3}{c}{ \texttt{SYN01}} & &  \multicolumn{3}{c}{\texttt{SYN02}} & \% increase\\
        \cmidrule{2-4} \cmidrule{6-8}
        & mean & & std. dev. & & mean & & std. dev. & $\left( \frac{\texttt{SYN02} - \texttt{SYN01}}{ \texttt{SYN01}} \right)$ \\
        \midrule
        \texttt{mspace-S$\mu$} &  299.18 & $\pm$ & 6.55 && 294.99 & $\pm$ & 8.81 & $-0.63$ \\
        \texttt{mspace-S$\gN$} &  400.99 & $\pm$ & 3.74 && 395.33 & $\pm$ & 3.24 & $-1.52$ \\
        \texttt{STGODE} &  420.86 & $\pm$ & 103.29 && 420.25 & $\pm$ & 52.17 & \textcolor{red}{$-9.87$} \\
        \texttt{GRAM-ODE} &  921.94 & $\pm$ & 537.63 && 853.77 & $\pm$ & 340.45 & \textcolor{red}{$-18.18$} \\
        \texttt{LightCTS} &  419.43 & $\pm$ & 176.5 && 334.59 & $\pm$ & 79.01 & \textcolor{red}{$-30.6$} \\
        \texttt{Kalman-$\vx$} &  781.94 & $\pm$ & 32.35 && 776.75 & $\pm$ & 30.38 & $-0.88$ \\
        \texttt{Kalman-$\vep$} &  393.76 & $\pm$ & 4.72 && 390.45 & $\pm$ & 3.54 & $-1.13$ \\
        \bottomrule
    \end{tabular}
    \label{tab:periodic}
\end{table}

\paragraph{Training Samples}
The generator parameters for \texttt{SYN03} and \texttt{SYN04} are same except for the total number of samples being ten times more in \texttt{SYN04}. If a model perform better on \texttt{SYN04} compared to \texttt{SYN03}, it would indicate that it is training intensive, requiring more samples to infer the trends. On the other hand, if the model performs worse on \texttt{SYN04}, it would indicate that there are scalability issues, or the training caused overfitting. An ideal model is expected to have similar performance on \texttt{SYN03} and \texttt{SYN04}. The models with ideal behaviour are marked \textcolor{teal}{teal}, and the models susceptible to overfitting are marked \textcolor{red}{red}. Moreover, model(s) that require more training samples are marked \textcolor{violet}{violet}.

\begin{table}[h!]
    \centering
    \caption{Impact of number of training samples on RMSE achieved by different models.}
    \begin{tabular}{lrclcrclc}
        \toprule
        & \multicolumn{3}{c}{ \texttt{SYN03}} & &  \multicolumn{3}{c}{\texttt{SYN04}} & \% increase\\
        \cmidrule{2-4} \cmidrule{6-8}
        & mean & & std. dev. & & mean & & std. dev. & $\left( \frac{\texttt{SYN04} - \texttt{SYN03}}{ \texttt{SYN03}} \right)$ \\
        \midrule
        \texttt{mspace-S$\mu$} & 793.41 & $\pm$ & 5.86 & & 789.36 & $\pm$ & 3 & \textcolor{teal}{$-0.86$} \\
        \texttt{mspace-S$\gN$} & 793.93 & $\pm$ & 5.73 & & 792.61 & $\pm$ & 2.02 & \textcolor{teal}{$-0.63$} \\
        \texttt{STGODE} & 830.63 & $\pm$ & 127 & & 931.33 & $\pm$ & 191.87 & \textcolor{red}{$+17.29$} \\
        \texttt{GRAM-ODE} & 1382.48 & $\pm$ & 80.78 & & 1423.93 & $\pm$ & 190.13 & \textcolor{red}{$+10.31$} \\
        \texttt{LightCTS} & 769.34 & $\pm$ & 196.6 & & 998.01 & $\pm$ & 319.72 & \textcolor{red}{$+36.42$} \\
        \texttt{Kalman-$\vx$} & 785.7 & $\pm$ & 8.95 & & 721.88 & $\pm$ & 1.73 & \textcolor{violet}{$-8.94$} \\
        \texttt{Kalman-$\vep$} & 782.6 & $\pm$ & 6.5 & & 783.36 & $\pm$ & 1.45 & \textcolor{teal}{$-0.54$} \\
        \bottomrule
    \end{tabular}
    \label{tab:timesteps}
\end{table}

\section{Evaluation}
\label{app:eval}

\subsection{Metrics}
The root mean squared error (RMSE) of $q$ consecutive predictions for all the nodes is:
\begin{align}
    \textstyle \textrm{RMSE}(q) \triangleq \mathbb{E} \left[\sqrt{\frac{1}{ndq}\sum_{v \in \gV} \sum_{i \in [q]} \lnorm \sum_{j \in [i]} \vep_{t+j}(v) - \hat{\vep}_{t+j}(v) \rnorm_2^2 } \right].
\end{align}
The  mean absolute error (MAE) of $q$ consecutive predictions for all the nodes is:
\begin{align}
    \textstyle \textrm{MAE}(q) \triangleq\frac{1}{ndq}\mathbb{E} \left[\sum_{v \in \gV}  \sum_{i \in [q]} \lnorm \sum_{j \in [i]} \vep_{t+j}(v) - \hat{\vep}_{t+j}(v) \rnorm_1\right].
\end{align}

\subsection{Datasets}
\label{sec:datasets}
In Table~\ref{tab:datasets}, we list the datasets commonly utilised in the literature for single and multi-step node feature forecasting.
\paragraph{tennis \cite{beres2018temporal}} represents a discrete-time dynamic graph showing the hourly changes in the interaction network among Twitter users during the 2017 Roland-Garros (RG17) tennis match. The input features capture the structural attributes of the vertices, with each vertex symbolizing a different user and the edges indicating retweets or mentions within an hour \footnote{\href{https://github.com/ferencberes/online-centrality}{https://github.com/ferencberes/online-centrality}}.
\paragraph{wikimath \cite{rozemberczki2021pytorch}} tracks daily visits to Wikipedia pages related to popular mathematical topics over a two-year period. Static edges denote hyperlinks between the pages \footnote{\href{https://pytorch-geometric-temporal.readthedocs.io/en/latest/_modules/torch_geometric_temporal/dataset/wikimath.html}{\texttt{wikimath} dataset from PyTorch Geometric Temporal}}. 
\paragraph{pedalme \cite{rozemberczki2021pytorch}} reports weekly bicycle package deliveries by Pedal Me in London throughout 2020 and 2021. The nodes are different locations, and the edge weight encodes the physical proximity. The count of weekly bicycle deliveries in a location forms the node feature footnote \footnote{\href{https://github.com/benedekrozemberczki/spatiotemporal_datasets}{https://github.com/benedekrozemberczki/spatiotemporal\_datasets}}.
\paragraph{cpox \cite{rozemberczki2021chickenpox}} tracks the weekly number of chickenpox cases for each county of Hungary between 2005 and 2015. Different counties form the nodes, and are connected if any two counties share a border  \footnotemark[\value{footnote}].

\paragraph{PEMS03/04/07/08 \cite{rao_fogs_2022}} 
The four datases are collected from four districts in California using the California Transportation Agencies (CalTrans) Performance Measurement System (PeMS) and aggregated into 5-minutes windows\footnote{\href{https://github.com/guoshnBJTU/ASTGNN/tree/main/data}{https://github.com/guoshnBJTU/ASTGNN/tree/main/data}}
. The spatial adjacency matrix for each dataset is constructed using the length of the roads. \texttt{PEMS03} is collected from September 2018 to November 2018. \texttt{PEMS04} is collected from San Francisco Bay area from July 2016 to August 2016. \texttt{PEMS07} is from Los Angeles and Ventura counties between May 2017 and August 2017. \texttt{PEMS08} is collected from San Bernardino area between July 2016 to August 2016.

\textit{Variables:} The \textbf{flow} represents the number of vehicles that pass through the loop detector per time interval (5 minutes). The \textbf{occupancy} variable represents the proportion of time during the time interval that the detector was occupied by a vehicle. It is measured as a percentage. Lastly, the \textbf{speed} variable represents the average speed of the vehicles passing through the loop detector during the time interval . It is measured in miles per hour (mph).

\paragraph{PEMSBAY \cite{li_diffusion_2017}} is a traffic dataset collected by CalTrans PeMS. It is represented by a network of 325 traffic sensors in the Bay Area with 6 months of traffic readings ranging from January 2017 to May 2017 in 5 minute intervals\footnote{\href{https://pytorch-geometric-temporal.readthedocs.io/en/latest/_modules/torch_geometric_temporal/dataset/pems_bay.html}{\texttt{PEMSBAY} dataset from PyTorch Geometric Temporal}}.

\paragraph{METRLA \cite{li_diffusion_2017}} is a traffic dataset based on Los Angeles Metropolitan traffic conditions. The traffic readings are collected from 207 loop detectors on highways in Los Angeles County over 5 minute intervals between March 2012 to June 2012\footnote{\href{https://pytorch-geometric-temporal.readthedocs.io/en/latest/_modules/torch_geometric_temporal/dataset/metr_la.html}{\texttt{METRLA} dataset from PyTorch Geometric Temporal}}.

\begin{table}[h!]
\centering
\caption{Real world datasets for single and multi-step forecasting.}
\small
\begin{tabular}{@{}lccccc@{}}
\toprule
\textbf{Name} & $n$ & $\vx$ & time-step & $T$ \\
\midrule
\texttt{tennis} & 1,000 &  \# tweets & 1 hour & 120 \\
\texttt{wikimath} & 1,068 &  \# visits & 1 day & 731 \\
\texttt{pedalme} & 15 &  \# deliveries & 1 week & 35 \\
\texttt{cpox} & 20&  \# cases & 1 week & 520 \\
\midrule
\texttt{PEMS03} & 358 &   flow & 5 min & 26,208 \\
\texttt{PEMS04} & 307 &   flow, occupancy, speed & 5 min & 16,992 \\
\texttt{PEMS07} & 883 &  flow & 5 min & 28,224 \\
\texttt{PEMS08} & 170 &  flow,  occupancy, speed & 5 min & 17,856 \\
\texttt{PEMSBAY} & 325 &  speed & 5 min & 52,116 \\
\texttt{METRLA} & 207 &  speed & 5 min & 34,272 \\
\bottomrule
\end{tabular}
\label{tab:datasets}

\end{table}

\subsection{Baselines}
\label{app:baselines}

\paragraph{DCRNN \cite{li_diffusion_2017}} The Diffusion Convolutional Recurrent Neural Network (\texttt{DCRNN}) models the node features as a diffusion process on a directed graph, capturing spatial dependencies through bidirectional random walks. Additionally, it addresses nonlinear temporal dynamics by employing an encoder-decoder architecture with scheduled sampling.

\paragraph{TGCN \cite{zhao_t-gcn_2019}} Temporal Graph Convolutional Network (\texttt{TGCN}) combines the graph convolutional network (GCN) with a gated recurrent unit (GRU), where the former learns the spatial patterns, and the latter learns the temporal.

\paragraph{EGCN \cite{pareja_evolvegcn_2020}} EvolveGCN (\texttt{EGCN}) adapts a GCN model without using node embeddings. The evolution of the GCN parameters is learnt through an RNN. \texttt{EGCN} has two variants: \texttt{ECGN-H} which uses a GRU, and \texttt{ECGN-O} which uses an LSTM.

\paragraph{DynGESN \cite{micheli_discrete-time_2022}} Dynamic Graph Echo State Networks (\texttt{DynGESN}) employ echo state networks (ESNs) a special type of RNN in which the recurrent weights are conditionally initialized, while a memory-less readout layer is trained. The ESN evolves through state transitions wheere the states belong to a compact space. For more details please refer to the original text.

\paragraph{GWNet \cite{wu_graph_2019} } GraphWave Net (\texttt{GWNet}) consists of an adaptive dependency matrix which is learnt through node embeddings, which is capable of capturing the hidden spatial relations in the data. \texttt{GWNet} can handle long sequences owing to its one-dimensional convolutional component whose receptive field grows exponentially with the number of layers.

\paragraph{STGODE \cite{fang_spatial-temporal_2021}} Spatial-temporal Graph Ordinary Differential Equation (\texttt{STGODE}) employs tensor-based ordinary differential equations (ODEs) to model the temporal evolution of the node features.

\paragraph{GRAM-ODE \cite{liu_graph-based_2023}} Graph-based Multi-ODE (\texttt{GRAM-ODE}) improves upon \texttt{STGODE} by connecting multiple ODE-GNN modules to capture different views of the local and global spatiotemporal dynamics.

\paragraph{FOGS \cite{rao_fogs_2022}} \texttt{FOGS} utilises first-order gradients to train a predictive model because the traffic data distribution is irregular.

\paragraph{LightCTS \cite{lai_lightcts_2023}} \texttt{LightCTS} stacks temporal and spatial operators in a computationally-efficient manner, and uses lightweight modules L-TCN and GL-Former.

\paragraph{ARIMA \cite{box_distribution_1970}} \texttt{ARIMA} is a multivariate time series forecasting technique that combines autoregressive, integrated, and moving average components. It models the relationship between observations and their lagged values, adjusts for non-stationarity in the data, and accounts for short-term fluctuations. 

\paragraph{Kalman \cite{welch_introduction_1997}} Since \texttt{mspace} is a state-space algorithm, we also use the Kalman filter  as a baseline. We introduce two variants of the Kalman filter: \texttt{Kalman-$\vx$}, which considers the node features as observations, and \texttt{Kalman-$\vep$}, which operates on the shocks.

\end{document}